\documentclass{article}
\usepackage[numbers]{natbib}%
\usepackage[final]{neurips_2020}
\usepackage{natbib}

\usepackage[utf8]{inputenc} 
\usepackage[T1]{fontenc}    
\usepackage{hyperref}       
\usepackage{paralist}
\usepackage{bbding}
\usepackage{nicefrac}       
\usepackage{microtype}      
\usepackage{url}            
\usepackage{booktabs}       
\usepackage{amsfonts}       
\usepackage{nicefrac}       
\usepackage{microtype}      
\usepackage{xcolor}
\usepackage{wrapfig}
\usepackage{algorithm,algorithmic}
\usepackage{standalone}     
\usepackage{preview} 
\usepackage{caption2}
\usepackage{dsfont}
\usepackage{multirow}
\usepackage{amsmath,mathrsfs,amssymb}
\usepackage{bm}
\usepackage{graphicx,subfigure,color} 
\usepackage{enumitem}
\usepackage{lipsum}
\usepackage{amsthm}
\usepackage{multicol}
\allowdisplaybreaks
\allowdisplaybreaks

\def \E {\mathbb{E}}
\def \x {\mathbf{x}}

\def \indicator {\mathds{1}}

\def \Y {\mathcal{Y}}

\def \X {\mathcal{X}}

\DeclareMathOperator*{\argmax}{arg\,max}
\DeclareMathOperator*{\argmin}{arg\,min}
\renewcommand{\tilde}{\widetilde}
\renewcommand{\hat}{\widehat}

\usepackage{mathtools}
\usepackage{appendix}
\let\norm\undefined 
\DeclarePairedDelimiter\norm{\lVert}{\rVert}

\newtheorem{myThm}{Theorem}

\newtheorem{myProp}{Proposition}

\theoremstyle{definition}

\newtheorem{myDef}{Definition}

\newtheorem{myRemark}{Remark}

\title{An Unbiased Risk Estimator for Learning with Augmented Classes}

\author{ Yu-Jie Zhang, Peng Zhao, Lanjihong Ma, Zhi-Hua Zhou\\
National Key Laboratory for Novel Software Technology, \\
Nanjing University, Nanjing 210023, China\\
\texttt{\{zhangyj, zhaop, maljh, zhouzh\}@lamda.nju.edu.cn} \\
}

\begin{document}

\maketitle

\begin{abstract}
This paper studies the problem of learning with augmented classes (LAC), where augmented classes unobserved in the training data might emerge in the testing phase. Previous studies generally attempt to discover augmented classes by exploiting geometric properties, achieving inspiring empirical performance yet lacking theoretical understandings particularly on the generalization ability. In this paper we show that, by using unlabeled training data to approximate the potential distribution of augmented classes, an unbiased risk estimator of the testing distribution can be established for the LAC problem under mild assumptions, which paves a way to develop a sound approach with theoretical guarantees. Moreover, the proposed approach can adapt to complex changing environments where augmented classes may appear and the prior of known classes may change simultaneously. Extensive experiments confirm the effectiveness of our proposed approach.

\end{abstract}

\section{Introduction}
\label{sec:intro}
Recent advances in machine learning encourage its application in high-stake scenarios, where the robustness is the central requirement~\citep{journals/aim/Dietterich17,journals/fcsc/Zhou16a}. A robust learning system should be able to handle the distribution change in the non-stationary environments~\citep{book/sugiyama/2012,journals/csur/GamaZBPB14,TKDE'19:DFOP}. In this paper, we focus on the problem of learning with augmented classes (LAC)~\citep{conf/aaai/DaYZ14}, where the class distribution changes during the learning process---some augmented classes unobserved in training data might emerge in testing. To make reliable predictions, desired learning systems are required to identify augmented classes and retain good generalization performance over the testing distribution.

The main challenge of the LAC problem lies in how to depict relationships between known and augmented classes. A typical solution is to learn a compact geometric description of the known classes and take those beyond the description as augmented classes, where the anomaly detection or novelty detection approaches can be employed (such as one-class SVM~\citep{journals/neco/ScholkopfPSSW01,journals/ml/TaxD04}, kernel density estimation~\citep{parzen-1962estimation,journals/jmlr/KimS12} and iForest~\citep{conf/icdm/LiuTZ08}).~\citet{conf/aaai/DaYZ14} give the name of LAC and employ the low-density separation assumption to adjust the decision boundaries in a multi-class situation. In addition to the effort of machine learning community, the computer vision and pattern recognition communities also contribute to the study of the problem (or its cousin).~\citet{Scheirer_2013_TPAMI} propose the notion of open space risk to penalize predictions outside the support of training data, based on which several approaches are developed~\citep{Scheirer_2013_TPAMI,journals/pami/ScheirerJB14}. Later, approaches based on the nearest neighbor~\citep{journals/ml/Mendes-JuniorSW17} and extreme value theory~\citep{journals/pami/RuddJSB18} are also developed. More discussions on related topics are deferred to Section~\ref{sec:related-work}. 

Although various approaches are proposed with nice performance and some of them conduct theoretical analysis, generalization properties of the LAC problem is less explored.~\citet{Scheirer_2013_TPAMI,journals/pami/ScheirerJB14,journals/pami/RuddJSB18} formally use the open space risk or extreme value theory to identify augmented classes, but the generalization error of learned models is not further analyzed. There are also works~\citep{journals/jmlr/ScottB09,journals/jmlr/BlanchardLS10,conf/icml/LiuGDFH18} focusing on the Neyman-Pearson (NP) classification, which controls the novelty detection ratio of augmented classes or false positive ratio of known classes with the constraint on another. By using unlabeled data, authors develop approaches with one-side PAC-style guarantees for the binary NP classification whereas the generalization ability for the LAC problem is not studied.

To design approaches with generalization error guarantees for the LAC problem, it is necessary to assess the distribution of augmented classes in the training stage. Note that in many applications, during the training stage, in addition to labeled data, there are abundant unlabeled training data available. In this paper, we show that by exploiting unlabeled training data, an \emph{unbiased} risk estimator over the \emph{testing} distribution can be established under mild assumptions. The intuition is that, though instances from augmented classes are unobserved from labeled data, their distribution information may be contained in unlabeled data and estimated by separating the distribution of known classes from unlabeled data (Figure~\ref{fig:approximation}). More concretely, we propose the \emph{class shift condition} to model the testing distribution as a mixture of known and augmented classes' distributions. Under such a condition, classifiers' risk over testing distribution can be estimated in the training stage, where minimizing its empirical estimator finally gives our \textsc{Eulac} approach, short for \underline{E}xploiting \underline{U}nlabeled data for \underline{L}earning with \underline{A}ugmented \underline{C}lasses. Moreover, the \textsc{Eulac} approach can further take the prior change on known classes into account, which enables its adaptivity to complex changing environments.

\begin{figure}[!t]
\centering
\includegraphics[clip, trim=3.27cm 7.1cm 0.55cm 3.8cm, width=0.8\textwidth]{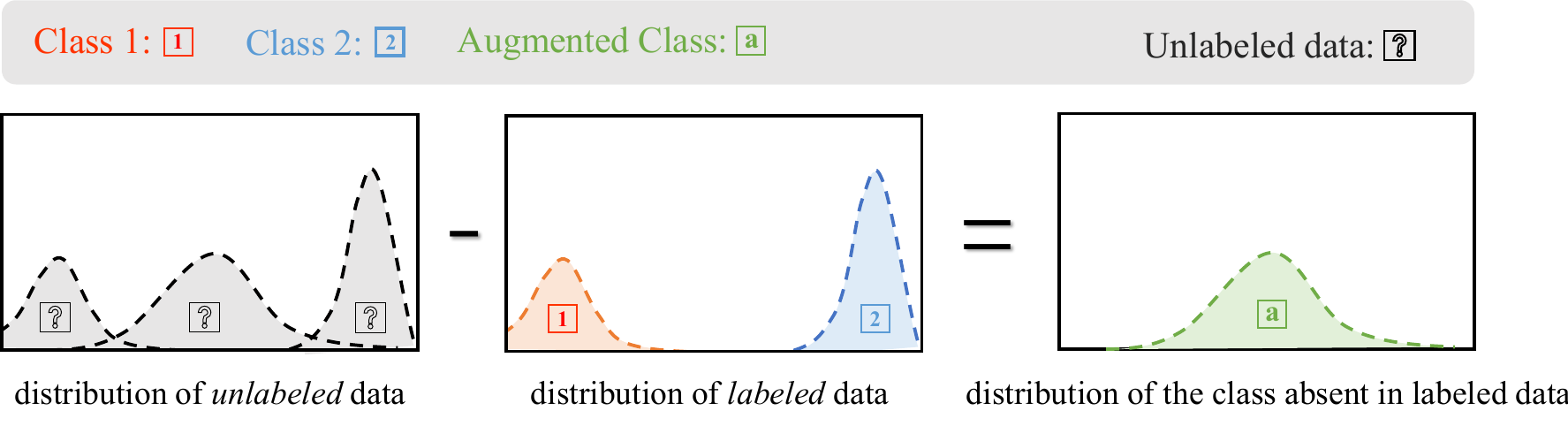}
\vspace{-3mm}
\caption{\small{Distribution of augmented classes can be estimated by those of labeled and unlabeled training data.}}
\vspace{-5mm}
\label{fig:approximation}
\end{figure}

\textsc{Eulac} enjoys several favorable properties. Theoretically, our approach has both asymptotic (consistency) and non-asymptotic (generalization error bound) guarantees. Notably, the non-asymptotic analysis further justifies the capability of our approach in exploiting unlabeled data, since the generalization error becomes smaller with an increasing number of unlabeled data. Moreover, extensive experiments validate the effectiveness of our approach. It is noteworthy to mention that our approach can now perform the standard cross validation procedure to select parameters, while most geometric-based approaches cannot due to the unavailability of the testing distribution, and their parameters setting heavily relies on the experience. We summarize main contributions of this paper as follows.
 \begin{compactitem}
     \item[(1)] We propose the \emph{class shift condition} to characterize the connection between known and augmented classes for the learning with augmented class problem.
     \item[(2)] Based on the class shift condition, we establish an \emph{unbiased risk estimator} over the testing distribution for the LAC problem by exploiting the unlabeled data. Similar results are also attainable for a general setting of class distribution change.
     \item[(3)] We develop our \textsc{Eulac} approach with the unbiased risk estimator, whose theoretical effectiveness is proved by both consistency and generalization error analyses. We also conduct extensive experiments to validate its empirical superiority.
 \end{compactitem}

\section{An Unbiased Risk Estimator for LAC problem} 
In this section, we formally describe the LAC problem, followed by the introduction of the class shift condition, based on which we develop the unbiased risk estimator over the testing distribution. Moreover, we show the potential of our approach for adapting to complex changing environments.

\subsection{Problem Setup and Class Shift Condition}
\textbf{LAC problem.} In the training stage, the learner collects a labeled dataset $D_{L}=\{(\mathbf{x}_i,y_i)\}_{i=1}^{n_l}$ sampled from distribution of known classes $P_\textsf{kc}$ defined over $\mathcal{X}\times\mathcal{Y}'$, where $\X$ denotes the feature space and $\mathcal{Y}'=\{1,\dots,K\}$ is the label space of $K$ known classes. In the testing stage, the learner requires to predict instances from the testing distribution $P_{te}$, where augmented classes not observed before might emerge. Since the specific partition of augmented classes is unknown, the learner will predict all of them as a single augmented class \textsf{ac}. So the testing distribution is defined over $\X\times\Y$, where $\Y = \{1,\dots,K,\textsf{ac}\}$ is the augmented label space. The goal of the learner is to train a classifier $f:\X\mapsto\Y$ achieving good generalization ability by minimizing the expected risk $R(f) = \mathbb{E}_{(\mathbf{x},y)\sim P_{te}}\ [\indicator(f(\x)\neq y)]$ over the \emph{testing} distribution, where $\indicator(\cdot)$ is the indicator function. 

\textbf{Unlabeled data.} In our setup, the learner additionally receives a set of \emph{unlabeled data} $D_U=\{\x_i\}_{i=1}^{n_u}$ sampled from the testing distribution and hopes to use it to enhance performance of the trained classifier. This learning scenario happens when labeled training data fail to capture certain classes of the testing distribution due to the class distribution change, while we can easily collect a vast amount of unlabeled data from current environments. Essentially, the missed class information has already been contained in the training data (unlabeled data) though is not revealed in the supervision (labeled training data). We thus prefer to call such classes as the ``augmented class'' instead of ``new class''. 

\textbf{Class shift condition.} Although not explicitly stated, previous works~\citep{Scheirer_2013_TPAMI,conf/aaai/DaYZ14,journals/ml/Mendes-JuniorSW17} essentially rely on the assumption that the distribution of known classes remains unchanged when augmented classes emerge. Following the same spirit, we introduce the following \emph{class shift condition} for the LAC problem to rigorously depict the connection between known and augmented class distributions.

\begin{myDef}[Class Shift Condition]
  \label{assumption:new class}
 The testing distribution $P_{te}$, the distribution of known classes $P_\textsf{kc}$ and the distribution of augmented classes $P_{\textsf{ac}}$ are under the \emph{class shift condition}, if
  \begin{equation}
  \label{eq:basic-assumption}
  P_{te} = \theta\cdot P_\textsf{kc} + (1-\theta)\cdot P_{\textsf{ac}},
  \end{equation}
  where $\theta\in[0,1]$ is a certain mixture proportion.\footnote{We redefine all the distributions over the space $\X\times\Y$, where $p_{XY}^\textsf{kc}(\x,\textsf{ac}) = 0$ for all $\x\in\X$}
\end{myDef}
Class shift condition states that the testing distribution can be regarded as a mixture of those of known and augmented classes with a certain proportion $\theta$, based on which we can evaluate classifiers' risk over the testing distribution with labeled and unlabeled training data.

\subsection{Convex Unbiased Risk Estimator}
\label{sec:approximating testing distribution}
This part, we develop an unbiased risk estimator for the LAC problem under the class shift condition. We first introduce the notation conventions. The density function is denoted by the lowercase $p$, and the joint, conditional and marginal density functions are indicated by the subscripts $XY$, $X\vert Y$ ($Y\vert X$) and $X$ ($Y$). For instance, $p_{X}^{te}(\x)$ refers to the marginal density of the testing distribution over $\mathcal{X}$.

\textbf{OVR scheme.} Suppose the joint testing distribution were available, the LAC problem would degenerate to standard multi-class classification, which can be then addressed by existing approaches. Among those approaches, we adopt the one-versus-rest (OVR) strategy, which enjoys sound theoretical guarantees~\citep{journals/jmlr/Zhang04a} and nice practical performance~\citep{journals/jmlr/RifkinK03}. The risk minimization is formulated as,
\begin{equation}
\label{eq:OVR-estimator}
 \min_{f_1,\dots,f_{K+1}} R_{\psi}(f_1,\dots,f_{K+1}) = \mathbb{E}_{(\mathbf{x},y)\sim P_{te}}\Big[\psi(f_y(\mathbf{x}))+\sum\nolimits_{k=1,k\neq y}^{K+1}\psi(-f_k(\mathbf{x}))\Big],
\end{equation} 
where $f_{k}$ is the classifier for the $k$-th class, $k=1,\ldots,K$; and $f_{\textsf{ac}}$ is the classifier for the augmented class. For simplicity, we substitute $f_{\textsf{ac}}$ with $f_{K+1}$ in the formulation. $\psi:\mathbb{R}\mapsto [0,+\infty)$ is a binary surrogate loss such as hinge loss. The OVR scheme predicts by $f(\x) = \argmax_{k \in \{1,\dots,K,\textsf{ac}\}} f_k(\x)$.

\textbf{Approximating the testing distribution.} However, \emph{the joint testing distribution is unavailable in the training stage due to the absence of labeled instances from augmented classes}. Fortunately, we show that given the mixture proportion $\theta$, it can be approximated with the labeled and unlabeled data. Under the class shift condition, the joint density of the testing distribution can be decomposed as
\begin{align}
    p^{te}_{XY}(\mathbf{x},y) \overset{\eqref{eq:basic-assumption}}{=}~& \theta\cdot p^\textsf{kc}_{XY}(\mathbf{x},y) + (1-\theta)\cdot p^{\textsf{ac}}_{XY}(\mathbf{x},y)\nonumber\\
  =~& \theta\cdot p^\textsf{kc}_{XY}(\mathbf{x},y) + \indicator(y=\textsf{ac})\cdot(1-\theta)\cdot p^{\textsf{ac}}_X(\mathbf{x})\label{eq:EULAC-decompose},
\end{align}
where the last equality follows from the fact that $p^{\textsf{ac}}_{XY}(\x,y) = 0$ holds for all $\x\in\X$ and $y\neq \textsf{ac}$. The first part $p_{XY}^\textsf{kc}(\mathbf{x},y)$ is accessible via the labeled data. The only unknown term is the second part, the marginal density of the augmented class $p_{X}^{\textsf{ac}}(\mathbf{x})$. Under the class shift condition, it can be evaluated by separating the distribution of labeled data from unlabeled data as
    \begin{equation}
    \label{eq:density-reduction}
        (1-\theta)\cdot p^{\textsf{ac}}_X(\mathbf{x}) =  p^{te}_X(\mathbf{x})-\theta\cdot p^\textsf{kc}_X(\mathbf{x}).
    \end{equation}
Thus, by plugging~\eqref{eq:density-reduction} into~\eqref{eq:EULAC-decompose}, the testing distribution becomes attainable, and consequently, we can evaluate the OVR risk $R_{\psi}$ in the training stage through an equivalent risk $R_{LAC}$.
\begin{myProp} \label{prop:LACU-equal-nonconvex}
Under the class shift condition, for measurable functions $f_1,\dots,f_K,f_{\textsf{ac}}$, we have $R_{\psi}(f_1,\dots,f_K,f_{\textsf{ac}}) = R_{LAC}(f_1,\dots,f_K,f_{\textsf{ac}})$, where $R_{LAC}$ is defined as,
\begin{equation}
\label{eq:expect-PU-nonconvex}
\begin{split}
R_{LAC} = {}& \theta\cdot\mathbb{E}_{(\mathbf{x},y)\sim P_\textsf{kc}}\left[\psi(f_y(\mathbf{x}))-\psi(-f_y(\mathbf{x})) + \psi(-f_{\textsf{ac}}(\mathbf{x}))-\psi(f_{\textsf{ac}}(\mathbf{x}))\right] \\
& + \mathbb{E}_{\mathbf{x}\sim p^{te}_X(\x)}\Big[\psi(f_{\textsf{ac}}(\mathbf{x}))+\sum\nolimits_{k=1}^K\psi(-f_{k}(\mathbf{x}))\Big].
\end{split}
\end{equation}
\end{myProp}
\begin{myRemark}
  We can assess $R_{LAC}$ during training as the distribution of known classes $ P_\textsf{kc}$ and marginal testing distribution $p_X^{te}(\x)$ can be estimated by labeled and unlabeled training data, respectively.
\end{myRemark}
The remaining issue for the LAC risk $R_{LAC}$ is the non-convexity caused by terms $-\psi(-f_y(\x))$ and $-\psi(f_{\textsf{ac}}(\x))$, which are non-convex w.r.t the classifiers even with the convex binary surrogate loss $\psi$. Inspired by studies~\citep{conf/nips/NatarajanDRT13,conf/icml/PlessisNS15}, we can eliminate the non-convexity by carefully choosing the surrogate loss satisfying $\psi(z) - \psi(-z) = -z$ for all $z\in\mathbb{R}$, and thereby $R_{LAC}$ enjoys a convex formulation
\begin{equation}
\label{eq:expect-PU}
R_{LAC} = \theta\cdot\mathbb{E}_{(\mathbf{x},y)\sim P_\textsf{kc}}\left[f_{\textsf{ac}}(\mathbf{x})-f_y(\mathbf{x})\right] + \mathbb{E}_{\mathbf{x}\sim p^{te}_{X}(\x)}\Big[\psi(f_{\textsf{ac}}(\mathbf{x}))+\sum\nolimits_{k=1}^K\psi(-f_{k}(\mathbf{x}))\Big].
\end{equation}
Many loss functions satisfy the above condition~\citep{conf/icml/PlessisNS15}, such as logistic loss $\psi(z) = \log(1+\exp(-z))$, square loss $\psi(z) = (1-z)^2/4$ and double hinge loss $\psi(z) = \max(-z, \max(0,(1-z)/2))$. Since LAC risk $R_{LAC}$ equals to the ideal OVR risk $R_{\psi}$, its empirical estimator $\hat{R}_{LAC}$ is \emph{unbiased} over the testing distribution. We can thus perform the standard empirical risk minimization. Finally, we note that Proposition~\ref{prop:LACU-equal-nonconvex} can be generalized for arbitrary multiclass losses, if the convexity is not required, where more multiclass and binary losses can be used. We will take this as a future work.

\subsection{Convex Unbiased Risk Estimator under Generalized Class Shift Condition}
The class shift condition in Definition~\ref{assumption:new class} models the appearance of augmented classes with the assumption that the distribution of known classes is identical to that in the testing stage. In real-world applications, however, the environments might be more complex, where the distribution of known classes could also shift. We consider a specific kind of class distribution change: in addition to the emerging augmented classes, the prior of each class $p^{te}_Y(y)$ varies from labeled data to testing data, while their conditional density remains the same, namely $p_{X\mid Y}^{te}(\x | y) = p_{X\mid Y}^\textsf{kc}(\x | y)$ for all $y\in[K]$. To this end, we propose following \emph{generalized class shift condition} to model such a case by further decomposing the distribution of known classes in the testing stage as a mixture of several components,
\begin{equation}
  \label{eq:basic-assumptio-target-shift}
P_{te} = \sum\nolimits_{k=1}^{K} \theta_{te}^k\cdot P_{k} + \bigg(1-\sum\nolimits_{k=1}^{K} \theta^k_{te}\bigg)\cdot P_{\textsf{ac}},
  \end{equation}
where $P_{k}$ is the distribution of the $k$-th known class whose marginal density equals to $p_{X| Y}^\textsf{kc}(\x|k)$, and $\theta_{te}^k = p_{Y}^{te}(k)$ is the prior of $k$-th known class in testing, for all $k \in [K]$. When there is no distribution change on known classes, the generalized class shift condition recovers the vanilla version in~\eqref{eq:basic-assumption}.

With the generalized class shift condition~\eqref{eq:basic-assumptio-target-shift}, following the similar argument in Section~\ref{sec:approximating testing distribution}, we can evaluate the OVR risk for the testing distribution even if the prior of known classes has changed.
\begin{myProp} \label{prop:LACU-equal-shift}
Under the generalized class shift condition~\eqref{eq:basic-assumptio-target-shift}, by choosing the surrogate loss function satisfying $\psi(z)-\psi(z) = -z$ for all $z\in\mathbb{R}$, for measurable functions $f_1,\dots,f_K,f_{\textsf{ac}}$, we have $R_{\psi}(f_1,\dots,f_K,f_{\textsf{ac}}) = R_{LAC}^{shift}(f_1,\dots,f_K,f_{\textsf{ac}})$, where $R_{LAC}^{shift}$ is defined as,
\begin{equation*}
\label{eq:expect-shift}
\begin{split}
R_{LAC}^{shift} = & \sum\nolimits_{k=1}^K\theta_{te}^k\cdot \mathbb{E}_{(\mathbf{x},y)\sim P_{k}}\left[f_{\textsf{ac}}(\mathbf{x})-f_y(\mathbf{x}) \right]  + \mathbb{E}_{\mathbf{x}\sim p^{te}_X(\x)}\Big[\psi(f_{\textsf{ac}}(\mathbf{x}))+\sum\nolimits_{k=1}^K\psi(-f_{k}(\mathbf{x}))\Big].
\end{split}
\end{equation*}
\end{myProp}
Proposition~\ref{prop:LACU-equal-shift} implies that we can handle the augmented classes together with the distribution change on prior of known classes by empirically minimizing the risk $R_{LAC}^{shift}$. Note that since $R_{LAC}^{shift}$ further decomposes the distribution of known classes into several components, it enjoys more flexibility than $R_{LAC}$ in evaluating the testing risk, yet requires more efforts in estimation of class prior $\theta_{te}^k$ for each known class rather than mixture proportion $\theta$ only, which will be discussed next.

\section{Approach}
\label{sec:algorithm}
In this section, we develop two practical algorithms for the proposed \textsc{Eulac} approach to minimize the empirical version of the LAC risk $R_{LAC}$ (similar results can be extended for its generalization $R_{LAC}^{shift}$). Meanwhile, we discuss how to estimate the mixture proportion $\theta$ and class prior $\theta_{te}^k$.

\textbf{Kernel-based hypothesis space.} We first consider minimizing the empirical LAC risk $\hat{R}_{LAC}$ in the reproducing kernel Hilbert space (RKHS) $\mathbb{F}$ associated to a PDS kernel $\kappa:\mathcal{X}\times\mathcal{X}\mapsto\mathbb{R}$:
\begin{equation}
\label{eq:LACU-ERM}
\min_{f_1,\dots,f_K,f_{\textsf{ac}}\in\mathbb{F}} \hat{R}_{LAC}+\lambda\Big(\sum\nolimits_{k=1}^K \Vert f_k\Vert^2_{\mathbb{F}}+\Vert f_{\textsf{ac}}\Vert^2_{\mathbb{F}}\Big),
\end{equation}
where $\hat{R}_{LAC}$ is the empirical approximation of the LAC risk~\eqref{eq:expect-PU}  
\begin{equation}
\label{eq:LACU-empirical risk}
\hat{R}_{LAC} = \frac{\theta}{n_l}\sum\nolimits_{i=1}^{n_l}\left(f_{\textsf{ac}}(\mathbf{x}_i)- f_{y_i}(\mathbf{x}_i)\right) + \frac{1}{n_u}\sum\nolimits_{i=1}^{n_u} \Big(\psi(f_{\textsf{ac}}(\mathbf{x}_i))+\sum\nolimits_{k=1}^K\psi(-f_{k}(\mathbf{x}_i))\Big).
\end{equation}
According to the representer theorem~\citep{book'01:Kernel}, the optimal solution of ~\eqref{eq:LACU-ERM} is provably in the form of
\begin{equation}
\label{eq:close-form-f}
f_k(\cdot) = \sum\nolimits_{\x_i\in D_{L}}\alpha^k_i \kappa(\cdot,\x_i)+\sum\nolimits_{x_j\in D_{U}}\alpha^k_j \kappa(\cdot,\x_j),
\end{equation}
where $\alpha_i^k$ is the $i$-th coefficient of the $k$-th classifier. Plugging~\eqref{eq:close-form-f} into~\eqref{eq:LACU-ERM}, we get a convex optimization problem with respect to $\bm{\alpha}$, which can be solved efficiently. Since the risk estimator $\hat{R}_{LAC}$ is assessed on the testing distribution directly, we can perform \emph{unbiased} cross validation to select parameters. Then, after obtaining the binary classifiers $f_1,\ldots,f_K, f_{\textsf{ac}}$, we follow the OVR rule to construct the final predictor as $f:\X\mapsto\Y$ with $f(\x)=\argmax\nolimits_{k \in \{1,\dots,K,\textsf{ac}\}}f_k(\x)$.

\textbf{Deep model.} Our approach can be also implemented by deep neural networks. Since the deep models themselves are non-convex, we  directly minimize the non-convex formulation of $R_{LAC}$~\eqref{eq:expect-PU-nonconvex} by taking outputs of the deep model as OVR classifiers. However, as shown by~\citet{conf/nips/KiryoNPS17}, the direct minimization easily suffers from over-fitting as the risk is not bounded from below by 0. To avoid the undesired phenomenon, we apply their proposed non-negative risk~\citep{conf/nips/KiryoNPS17} to rectify the OVR scheme for training the deep model, whose effectiveness will be validated by experiments. More detailed elaborations for the rectified $R_{LAC}$ risk is presented in Appendix~\ref{sec-appendix:deep}.

\textbf{On the estimation of $\theta$.} Notice that minimizing $\hat{R}_{LAC}$ requires estimating $\theta$, which is known as the problem of \emph{Mixture Proportion Estimation} (MPE)~\citep{conf/icml/RamaswamyST16}, where one aims to estimate the maximum proportion of distribution $H$ in distribution $F$ given their empirical observations. Many works have been devoted to developing theoretical foundations and efficient algorithms~\citep{conf/kdd/ElkanN08,journals/jmlr/BlanchardLS10,journals/corr/JainWTR16,journals/ml/PlessisNS17}. We employ the kernel mean embedding (KME) based algorithm proposed by~\citet{conf/icml/RamaswamyST16}, which guarantees that the estimator $\hat{\theta}$ converges to true proportion $\theta$ in the rate of $\mathcal{O}(1/\sqrt{\min\{n_l,n_u\}})$ under the separability condition. Moreover, since the KME-based algorithm easily suffers from the curse of dimensionality in practice, inspired by the recent work~\citep{journals/corr/JainWTR16}, we further use a pre-trained model to reduce the dimensionality of original input to its probability outputs. We present more details of the proportion estimation in Appendix~\ref{sec-appendix:MPE}. We refer to the above estimator as KME-base, and the corresponding approach for LAC as \textsc{Eulac}-base.

Additionally, under the generalized class shift condition, we need more refined estimations for each known class. Therefore, we use the above MPE estimator to estimate each class prior $\theta_{te}^k$ in $R_{LAC}^{shift}$~\eqref{eq:expect-shift} via the labeled instances from the $k$-th known class and the unlabeled data, $k \in [K]$. We refer to such an estimator as KME-shift and the corresponding approach as \textsc{Eulac}-shift. Finally, we note that since the vanilla LAC can also be modeled with the generalized class shift condition, we can use KME-shift to estimate the mixture proportion $\hat{\theta}$ by  $\hat{\theta} = \sum_{k=1}^K\hat{\theta}_{te}^k$. It turns out that KME-shift achieves comparable (even better) empirical performance with KME-base.

\section{Theoretical Analysis}
\label{sec:theory}
In this section, we first show the infinite-sample \emph{consistency} of the LAC risk $R_{LAC}$. Then, we derive the \emph{generalization error} bounds. All the proofs can be found in Appendix~\ref{sec-appendix:proof-all}.

\textbf{Infinite-sample consistency.}
\label{sec:consistency}
At first, we show that the LAC risk $R_{LAC}$ is infinite-sample consistent with the risk over the testing distribution with respect to 0-1 loss. Namely, by minimizing the expected risk of $R_{LAC}$, we can get classifiers achieving the Bayes rule over the testing distribution.
\begin{myThm}
\label{thm:consistent}
Under the class shift condition, suppose the surrogate loss  $\psi$ is convex, bounded below, differential, satisfying $\psi(z)-\psi(-z) =-z$ and $\psi(z)<\psi(-z)\mbox{ when } z>0$, then for any $\epsilon_1>0$, there exists $\epsilon_2>0$ such that
\[
  R_{LAC}(f_1,\dots,f_K,f_{\textsf{ac}})\leq R_{LAC}^* + \epsilon_2   \quad \Longrightarrow  \quad  R(f)\leq R^* + \epsilon_1
\]
holds for all measurable functions $f_1,\dots,f_K,f_{\textsf{ac}}$ and $f(\x) = \argmax_{k\in \{1,\dots,K,\textsf{ac}\}}f_k(\x)$. Here, $R_{LAC}^* = \min\nolimits_{f_1,\dots,f_K,f_{\textsf{ac}}} R_{LAC}(f_1,\dots,f_K,f_{\textsf{ac}})$ and $R^* = \min_f R(f) = \mathbb{E}_{(\mathbf{x},y)\sim P_{te}}\ [\indicator(f(\x)\neq y)]$ is the Bayes error over the testing distribution.
\end{myThm}
Theorem~\ref{thm:consistent} follows from Proposition~\ref{prop:LACU-equal-nonconvex} and analysis in the seminal work of~\citet{journals/jmlr/Zhang04a}, who investigates the consistency property of OVR risk in depth. Since the LAC risk $R_{LAC}$ is equivalent to the OVR risk $R_{\psi}$, it is naturally infinite-sample consistent. There are many loss functions satisfy assumptions in Theorem~\ref{thm:consistent} such as the logistic loss $\psi(z) = \log(1+\exp(-z))$ and the square loss $\psi(z) = (1-z)^2/4$. In particular, we can obtain a more quantitative results for the square loss.
\begin{myThm}
\label{thm:consistent-square loss}
Under the same condition of Theorem~\ref{thm:consistent}, when using $\psi(z) = (1-z)^2/4$ as the surrogate loss function, we have $R(f)-R^* \leq \sqrt{2\big(R_{LAC}(f_1,\dots,f_K,f_{\textsf{ac}})-R^*_{LAC}\big)}$.
\end{myThm}
Theorem~\ref{thm:consistent-square loss} shows that the excess risk of $R_{LAC}$ upper bounds that of 0-1 loss. Thus, by minimizing the LAC risk $R_{LAC}$, we can obtain well-behaved classifiers on the testing distribution w.r.t. 0-1 loss.

\begin{myRemark}
Theorems~\ref{thm:consistent} and~\ref{thm:consistent-square loss} show the consistency for $R_{LAC}$ under class shift condition. Similar results can be easily obtained for $R_{LAC}^{shift}$ with the generalized class shift condition, due to the equivalence of $R_{LAC}^{shift}$ and the OVR risk, even when prior of known classes have changed.
\end{myRemark}

\textbf{Finite-sample generalization error bound.}
\label{sec:excess-risk-bound}
We establish the generalization error bound for the proposed approach in this part. Since the approach actually minimizes the empirical risk estimator $\hat{R}_{LAC}$ with a regularization term of the RKHS $\mathbb{F}$, it is equivalent to investigate the generalization ability of classifiers $f_1,\dots,f_K,f_{\textsf{ac}}$ in the kernel-based hypothesis set $\mathcal{F} = \{\mathbf{x}\mapsto\langle\mathbf{w},\Phi(\mathbf{x})\rangle\ \vert\ \Vert\mathbf{w}\Vert_{\mathbb{F}}\leq\Lambda  \}$, where $\Phi:\mathbf{x}\mapsto\mathbb{F}$ is a feature mapping associated with the positive definite symmetric kernel $\kappa$, and $\mathbf{w}$ is an element in the RKHS $\mathbb{F}$. We have the following generalization error bound.
\begin{myThm}
\label{thm:generalization-bound}
Assume that $\kappa(\x,\x) \leq r^2$ holds for all $\x\in\mathcal{X}$ and the surrogate loss function $\psi$ is bounded by $B_\psi \geq 0$ and is $L$-Lipschitz continuous.\footnote{Common surrogate loss functions  satisfy these conditions, such as logistic loss, exp loss and square loss.} Then, for any $\delta>0$, with probability at least $1-\delta$ over the draw of labeled samples $D_{L}$ of size $n_l$ from the distribution of known classes $P_\textsf{kc}$ and unlabeled samples $D_U$ of size $n_u$ from $p_{X}^{te}(\x)$, the following holds for all $f_1,\dots,f_K,f_{\textsf{ac}}\in\mathcal{F}$,
\begin{align*}
         {} & R_{LAC}(f_1,\dots,f_K,f_{\textsf{ac}}) - \hat{R}_{LAC}(f_1,\dots,f_K,f_{\textsf{ac}})\\ 
    \leq {} & \frac{2(K+1)\Lambda r}{\sqrt{n_l}}+6\Lambda r\sqrt{\frac{2\log(4/\delta)}{n_l}} + \frac{2(K+1)L\Lambda r}{\sqrt{n_u}}+3(K+1)B_\psi\sqrt{\frac{\log(4/\delta)}{n_u}}.
\end{align*}
\end{myThm}
Based on Theorem~\ref{thm:generalization-bound}, by the standard argument~\citep{MLSS'03:STL,book'12:foundation}, we can obtain the estimation error bound. 
\begin{myThm}
    \label{thm:estimation-error}
    Under the same assumptions of Theorem~\ref{thm:generalization-bound} and let $\hat{f}_1,\dots,\hat{f}_K,\hat{f}_{\textsf{ac}}$ be the optimal solution of the optimization problem~\eqref{eq:LACU-ERM} with certain $\lambda>0$, with high probability, we have
    \begin{equation*}
        R_{LAC}(\hat{f}_1,\dots,\hat{f}_K,\hat{f}_{\textsf{ac}})- \inf_{\bm{f}\in\mathscr{F}} R_{LAC}(f_1,\dots,f_K,f_{\textsf{ac}}) \leq \mathcal{O}\left(\frac{K+1}{\sqrt{n_l}}+\frac{K+1}{\sqrt{n_u}}\right), 
    \end{equation*}
    where $\bm{f}$ denotes $(f_1,\dots,f_K,f_{\textsf{ac}})$ and $\mathscr{F}=\{\bm{f}\ \vert\  f_1,\dots,f_K,f_{\textsf{ac}}\in\mathbb{F},\sum_{k=1}^K\Vert f_k\Vert^2_\mathbb{F}+\Vert f_{\textsf{ac}}\Vert^2_\mathbb{F}\leq c_\lambda^2\}$. The parameter $c_\lambda>0$ is a constant related to $\lambda$ in~\eqref{eq:LACU-ERM}. We use the $\mathcal{O}$-notation to keep the dependence on $n_u$, $n_l$ and $K$ only, where the full expression can be found in Appendix~\ref{sec-appendix:proof-FSC}.
\end{myThm}

\begin{myRemark}
Theorem~\ref{thm:generalization-bound} and Theorem~\ref{thm:estimation-error} show that, the estimation error of the trained classifiers decreases with a growing number of labeled and \emph{unlabeled} data, which theoretically justifies the effecacy of our approach in exploiting unlabeled data. Experiments also validate the same tendency.
\end{myRemark}
\textbf{Overview of theoretical results.} Recall that the goal of the LAC problem is to obtain classifiers that approach Bayes rule over the testing distribution, so we need to minimize the excess risk $R\big(\mathop{\mathrm{argmax}}\nolimits_{k\in \{1,\dots,K,\textsf{ac}\}}f_k\big)-R^*$. According to the consistency guarantee presented in Theorem~\ref{thm:consistent}, it suffices to minimize the excess risk $R_{LAC}(\bm{f})-R_{LAC}^*$, which can be further decomposed into the estimation error and the approximation error as follows,
\begin{align*}
    R_{LAC}(\bm{f})-R_{LAC}^* = \underbrace{R_{LAC}(\bm{f}) - \inf\nolimits_{\bm{f}\in\mathscr{F}}R_{LAC}(\bm{f})}_{\mathtt{estimation~error}} + \underbrace{\inf\nolimits_{\bm{f}\in\mathscr{F}}R_{LAC}(\bm{f}) - R_{LAC}^*}_{\mathtt{approximation~error}}.
\end{align*}
Theorem~\ref{thm:estimation-error} shows that with an increasing number of labeled and unlabeled data, the excess risk converges to the irreducible approximation error, which measures how well the hypothesis set approximates the Bayes rule and is generally not accessible for learning algorithms~\citep{book'12:foundation}. Thus, the consistency and excess risk bounds theoretically justify the effectiveness of our approach.

\section{Related Work and Discussion}
\label{sec:related-work}
This section discusses several research topics and techniques that are related to our approach.

\textbf{Class-incremental learning}~\citep{journals/kbs/ZhouC02} aims to handle new classes appearing in the learning process, and learning with augmented classes is one of its core tasks. Some early studies~\citep{conf/aaai/DaYZ14,conf/icdm/DingLZ18} try to exploit unlabeled data for handling the LAC problem. Our approach differs from theirs as we depict the connection between known and augmented classes by the class shift condition rather than the geometric assumption, which leads to more clear theoretical understandings and better performance. Apart from the batch setting, researchers also manage to handle even more challenging scenario where augmented classes emerge in the streaming data~\citep{conf/icml/FinkSSU06,journals/tnn/MuhlbaierTP09,journals/tkde/MuTZ17,ICDM'19:SENNE}. It is interesting to study that whether our approach can be tailored for the streaming setting.

\textbf{Open set recognition} \citep{Scheirer_2013_TPAMI,arxiv:OSR-survey} is a cousin of the LAC problem studies in the computer vision and pattern recognition communities. As we have mentioned, several techniques or concepts are employed to depict the relationship between known and augmented classes, including open space risk~\citep{Scheirer_2013_TPAMI,journals/pami/ScheirerJB14}, nearest neighbor approach~\citep{journals/ml/Mendes-JuniorSW17}, extreme value theory~\citep{journals/pami/RuddJSB18} and the adversarial sample generation framework~\citep{conf/ijcai/YuQLG17}, etc. We note that many works in OSR implicitly use the feature semantic information to help identifying augmented classes. By contrast, our paper works on a general setting without such domain knowledge on the semantic information.

Although the approaches achieve nice empirical behavior and are underpinned by formal definitions or theories, their generalization error over testing distribution are less explored. Exceptions are works~\citep{journals/jmlr/ScottB09,journals/jmlr/BlanchardLS10,conf/icml/LiuGDFH18}. Authors focus on the Neyman-Pearson~(NP) classification problem, where false positive ratio on known classes are minimized with the constraint on desired novelty detection ratio, or vice.~\citet{journals/jmlr/ScottB09} and~\citet{journals/jmlr/BlanchardLS10} provide one-side generalization bounds for both the novelty detection ratio and false positive ratio. However, the results mainly focus on the binary NP classification problem. The generalization error and excess risk analysis for the LAC problem, where multiple classes appear, is not investigated.~\citet{conf/icml/LiuGDFH18} design a general meta-algorithm that can take any existing novelty detection approach as a subroutine to recognize augmented classes. They contribute to the PAC-style guarantee for the meta-algorithm on the novelty detection ratio, while performance on the false positive rate is less explored.

\textbf{Learning with positive and unlabeled examples}~(LPUE), also known as PU learning, is a special semi-supervised learning task aiming to train a classifier for the binary classification with the positive and unlabeled data only~\citep{conf/icml/LiuLYL02,conf/kdd/ElkanN08,journals/biometrics/Ward09,conf/nips/PlessisNS14,conf/icml/PlessisNS15}. One research line of LPUE is to exploit the risk rewriting technique to establish unbiased estimators for classifier training, which have also been adopted in our paper. The LAC problem with unlabeled data can be seen as a generalized LPUE problem by taking the known classes as positive. However, most studies on LPUE mainly focus on the binary scenario and established approaches are no longer unbiased in the multiclass case. For the multiclass scenario,~\citet{conf/ijcai/XuX0T17} exploit the risk rewriting technique to train linear classifiers, which has also been adopted by~\citet{journals/corr/abs-1901-11351} for ordinal regression. Although sharing similarity with~\citep{conf/ijcai/XuX0T17,journals/corr/abs-1901-11351}, our LAC risk is established in a quite different context and brings novel understandings for the LAC, through which more complex changing environments could be handled. Besides, the LAC risk allows more flexible implementations where the kernel method and deep model are applicable.

\section{Experiments}
\label{sec:expr}
We examine three aspects of the proposed \textsc{Eulac} approach: (\textsf{Q1}) performance of classifying known classes and identifying augmented classes; (\textsf{Q2}) accuracy of estimating mixture prior $\theta$ and its influence on \textsc{Eulac}; (\textsf{Q3}) capability of handling the complex changing environments (augmented class appears and prior of known classes shifts simultaneously). We answer the questions in following three subsections. In all experiments, classifiers are trained with labeled and unlabeled data, and are evaluated with an additional testing dataset which is never observed in training.

\begin{table}[!t]
\centering
\setcaptionwidth{0.94\textwidth}
\caption{\small{Macro-F1 scores on benchmark datasets. The best method is emphasized in bold. Besides, $\bullet$ indicates that \textsc{Eulac} is significantly better than others (paired $t$-tests at 5\% significance level).}}
\label{table:experiment1-MacaroF1}
\vspace{-1mm}
\resizebox{0.94\textwidth}{!}{
\begin{tabular}{lllllllll}
\toprule
\multirow{1}{*}{Dataset} & \multicolumn{1}{c}{\multirow{1}{*}{OVR-SVM}} & \multicolumn{1}{c}{\multirow{1}{*}{W-SVM}} & \multicolumn{1}{c}{\multirow{1}{*}{OSNN}} & \multicolumn{1}{c}{\multirow{1}{*}{EVM}} & \multicolumn{1}{c}{\multirow{1}{*}{LACU-SVM}} & \multicolumn{1}{c}{PAC-iForest}      & \multicolumn{1}{c}{\multirow{1}{*}{\textsc{Eulac}}} \\ \midrule
usps                                        & 75.42 $\pm$ 4.87 $\bullet$         & 79.77 $\pm$ 4.97 $\bullet$       & 63.14 $\pm$ 8.91 $\bullet$      & 61.14 $\pm$ 6.27 $\bullet$     & 69.20 $\pm$ 8.34 $\bullet$           & 55.69 $\pm$ 13.3 $\bullet$     & \textbf{86.52 $\pm$ 2.72}       \\
segment                                      & 71.78 $\pm$ 5.12 $\bullet$         & 80.82 $\pm$ 9.38 $\bullet$       & 85.10 $\pm$ 5.98                & 82.13 $\pm$ 5.88 $\bullet$     & 40.69 $\pm$ 12.5 $\bullet$          & 63.64 $\pm$ 13.1 $\bullet$     & \textbf{86.17 $\pm$ 5.80}       \\
satimage                                     & 54.67 $\pm$ 9.80 $\bullet$         & 76.29 $\pm$ 13.2 $\bullet$       & 62.48 $\pm$ 11.2 $\bullet$      & 72.10 $\pm$ 8.16 $\bullet$     & 51.56 $\pm$ 17.3 $\bullet$          & 60.76 $\pm$ 7.79 $\bullet$     & \textbf{81.25 $\pm$ 6.18}       \\
optdigits                                    & 80.11 $\pm$ 3.80 $\bullet$         & 87.82 $\pm$ 4.64 $\bullet$       & 86.97 $\pm$ 3.79 $\bullet$      & 72.00 $\pm$ 8.33 $\bullet$     & 80.92 $\pm$ 3.68 $\bullet$          & 71.65 $\pm$ 5.46 $\bullet$     & \textbf{91.54 $\pm$ 2.95}       \\
pendigits                                    & 72.78 $\pm$ 5.19 $\bullet$         & 87.79 $\pm$ 3.95                 & 86.69 $\pm$ 3.39 $\bullet$      & \textbf{89.94 $\pm$ 1.30}     & 70.66 $\pm$ 6.18 $\bullet$           & 73.21 $\pm$ 4.52 $\bullet$     & 88.41 $\pm$ 4.81                 \\
SenseVeh                                     & 48.07 $\pm$ 3.80 $\bullet$         & 45.96 $\pm$ 2.32 $\bullet$       & 49.91 $\pm$ 6.88 $\bullet$      & 51.24 $\pm$ 3.91 $\bullet$     & 51.61 $\pm$ 3.31 $\bullet$          & 54.12 $\pm$ 7.19 $\bullet$     & \textbf{77.33 $\pm$ 2.17}       \\
landset                                      & 60.43 $\pm$ 7.65 $\bullet$         & 68.91 $\pm$ 17.0 $\bullet$       & 73.25 $\pm$ 9.23 $\bullet$      & 76.00 $\pm$ 7.79 $\bullet$     & 53.59 $\pm$ 9.88 $\bullet$          & 70.50 $\pm$ 7.16 $\bullet$     & \textbf{85.70 $\pm$ 4.46}       \\
mnist                                        & 66.74 $\pm$ 2.76 $\bullet$         & 75.38 $\pm$ 4.62 $\bullet$       & 57.75 $\pm$ 10.9 $\bullet$      &  58.39 $\pm$ 5.94 $\bullet$       & 63.53 $\pm$ 7.58 $\bullet$       & 48.31 $\pm$ 9.62 $\bullet$     & \textbf{80.66 $\pm$ 5.38}       \\
shuttle                                      & 37.39 $\pm$ 14.1 $\bullet$         & 58.48 $\pm$ 34.5 $\bullet$       & 48.21 $\pm$ 16.4 $\bullet$      &  \multicolumn{1}{c}{--}         & 34.18 $\pm$ 13.4 $\bullet$         & 29.36 $\pm$ 8.70 $\bullet$     & \textbf{66.49 $\pm$ 17.9}       \\ \midrule
\textsc{Eulac} w/ t/ l                                & \multicolumn{1}{c}{9/ 0/ 0}        & \multicolumn{1}{c}{8/ 1/ 0}      & \multicolumn{1}{c}{8/ 1/ 0}     & \multicolumn{1}{c}{8/ 1/ 0}       & \multicolumn{1}{c}{9/ 0/ 0}     & \multicolumn{1}{c}{9/ 0/ 0} & rank first 8/ 9       \\ 
\bottomrule
\end{tabular}
}
\vspace{-3mm}
\end{table}

\subsection{Performance Comparison} 
\label{sec:5.1}
To answer \textsf{Q1}, we compare two implementations of \textsc{Eulac} (RKHS-based and DNN-based versions) with contenders on several benchmark datasets for various tasks. The overall performance over testing distribution is measured by the Macro-F1 score and accuracy. Meanwhile, we report AUC of augmented class score to evaluate the ability of identifying augmented classes. Due to space constraints, we provide detailed descriptions of datasets, contenders and measures in Appendix~\ref{sec-appendix:experiments-1}.

\textbf{Comparison on RKHS-based Eulac.} We adopt 9 datasets, where half of the total classes are randomly selected as augmented classes for 10 times. In each dataset, the labeled, unlabeled and testing data contain 500, 1000 and 1000 instance respectively. The instance sampling procedure repeats 10 times. Meanwhile, there are six contenders, including four without exploiting unlabeled data (OVR-SVM, W-SVM~\citep{journals/pami/ScheirerJB14}, OSNN~\citep{journals/ml/Mendes-JuniorSW17}, EVM~\citep{journals/pami/RuddJSB18}) and two using them (LACU-SVM~\citep{conf/aaai/DaYZ14}, PAC-iForest~\citep{conf/icdm/LiuTZ08}). Table~\ref{table:experiment1-MacaroF1} reports performance in terms of Macro-F1 score. Similar results for accuracy and AUC are shown in Appendix~\ref{sec-appendix:experiments-1}. We can see that \textsc{Eulac} outperforms others in most datasets. Note that it is surprising that W-SVM and EVM achieve better results than LACU-SVM and PAC-iForest, which are fed with unlabeled data. The reason might be that these methods require to set parameters empirically and the default one may not be proper for all datasets. By contrast, our proposed \textsc{Eulac} can perform \emph{unbiased} cross validation to select proper parameters. 

\textbf{Influence on the size of unlabeled data.} We vary the size of unlabeled data from 250 to 1500 with an interval of 250 on 3 datasets: mnist, landset, and usps. LACU-SVM and PAC-iForest are included for comparison. Figure~\ref{figure:exp2-unlabeled} presents the Macro-F1 score and shows that the score of LACU-SVM remains unchanged or even drops in the three datasets, while performance of our approach is enhanced when provided with more unlabeled data, which is consistent with theoretical analysis in Section~\ref{sec:theory}. This again validates that our approach can exploit unlabeled data effectively. Notice that PAC-iForest also enjoys sound theoretical guarantees, yet the guarantees only hold for the novelty detection ratio and thus the overall performance on the testing distribution is not promised to be improved.

\begin{figure}[!t]
    \begin{minipage}[t]{\linewidth}
    \centering
 \begin{subfigure}[\textsf{mnist}]{
    \begin{minipage}[b]{0.3\textwidth}
    \label{figure:mnist-F1}
     \includegraphics[clip, width=\textwidth]{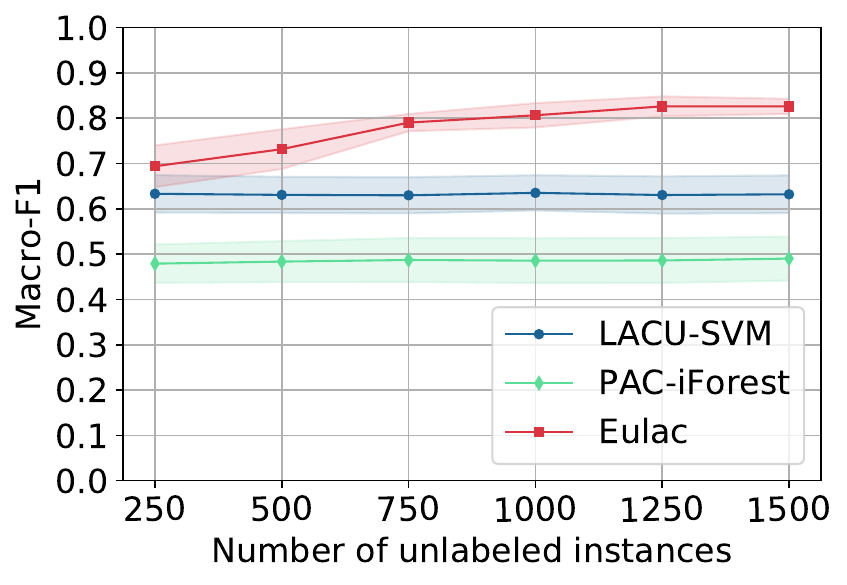} 
     \end{minipage}}
  \end{subfigure}
   \begin{subfigure}[\textsf{landset}]{
     \begin{minipage}[b]{0.3\textwidth}
     \label{figure:landset-F1}
      \includegraphics[clip, width=\textwidth]{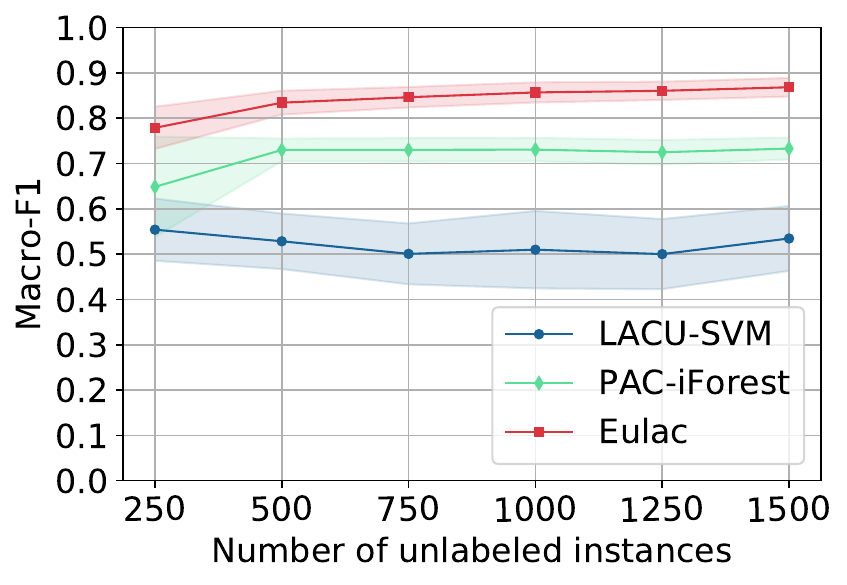}
      \end{minipage}}
   \end{subfigure}
  \begin{subfigure}[\textsf{usps}]{
    \begin{minipage}[b]{0.3\textwidth}
    \label{fig:propa}
     \includegraphics[clip, width=\textwidth]{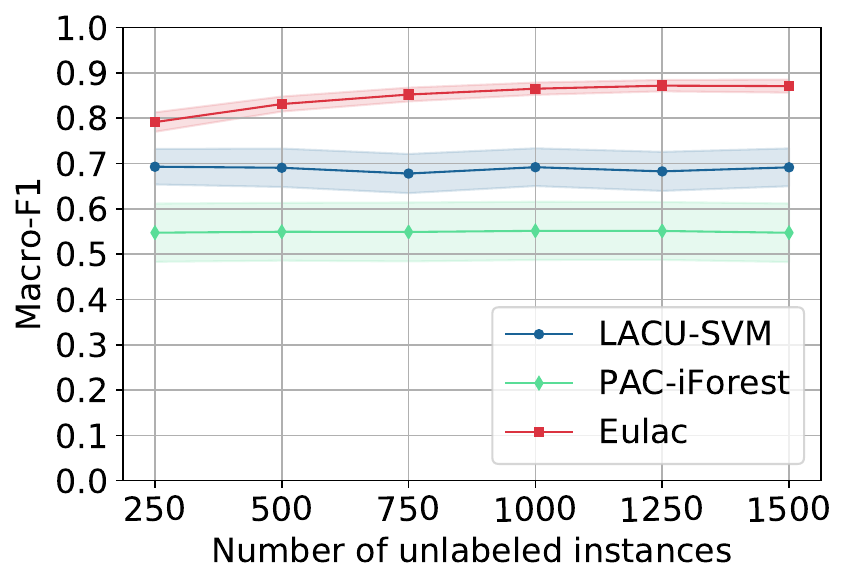}
     \end{minipage}}
  \end{subfigure}
  \vspace{-3mm}
  \caption{Macro-F1 score comparisons when the number of unlabeled data increases.}
  \label{figure:exp2-unlabeled}
  \vspace{-5mm}
  \end{minipage}  
\end{figure}

\begin{wraptable}{r}{6cm}
\vspace{-6mm}
  \centering
  \caption{AUC for DNN-based \textsc{Eulac}}
  \vspace{2mm}
           \label{table:AUC-deep}
     \resizebox{0.4\textwidth}{!}{
    \begin{tabular}{lccc}
    \toprule
    Methods     & mnist            &    Cifar-10  &    SVHN   \\ \midrule
    SoftMax     &  97.8 $\pm$ 0.6        &    67.7 $\pm$ 3.8     &    88.6 $\pm$ 1.4   \\
    OpenMax     &  98.1 $\pm$ 0.5         &    69.5 $\pm$ 4.4    &    89.4 $\pm$ 1.3  \\
    G-OpenMax   &  98.4 $\pm$ 0.5        &    67.5 $\pm$ 4.4    &    89.6 $\pm$ 1.7  \\
    OSRCI       &   \textbf{98.8 $\pm$ 0.4}         &    69.9 $\pm$ 3.8    &    \textbf{91.0 $\pm$ 1.0}  \\
    \textsc{Eulac}      &   98.6 $\pm$ 0.4                &   \textbf{85.2 $\pm$ 2.0}          &    \textbf{91.2 $\pm$ 2.8}       \\ \bottomrule
    \end{tabular}}
      \vspace{-3mm}
\end{wraptable}

\textbf{Comparison on deep models.} We also evaluate DNN-based \textsc{Eulac}, where the sigmoid loss $\psi(z) = 1/(1+\exp(z))$ is used for the non-negative risk. The experiments are conducted on mnist, SVHN and Cifar-10 datasets, where six of all ten classes are randomly selected as known while the rest four are treated as augmented. The contenders are SoftMax, OpenMax~\citep{conf/cvpr/BendaleB16}, G-OpenMax~\citep{conf/bmvc/GeDG17}, OSRCI~\citep{conf/eccv/NealOFWL18}. All methods are trained based on the standard training split. The unlabeled data are sampled from part of the standard testing split and the rest instances are used for evaluation. We present more details in Appendix~\ref{sec-appendix:experiments-2}. Following the previous study~\citep{conf/eccv/NealOFWL18}, we report AUC of the augmented class in Table~\ref{table:AUC-deep}, and results of contenders are also from~\citep{conf/eccv/NealOFWL18}. DNN-based \textsc{Eulac} can learn nice detection score for identifying augmented classes, which validates its efficacy.

\subsection{Issue of Mixture Proportion}
To answer \textsf{Q2}, we conduct experiments on mnist dataset, where the true mixture proportion varies from 0.1 to 0.9. Other configurations are the same as those in Section~\ref{sec:5.1}.

\begin{wrapfigure}{r}{5.4cm}
\vspace{-4mm}
\centering
\includegraphics[width = 0.37\textwidth]{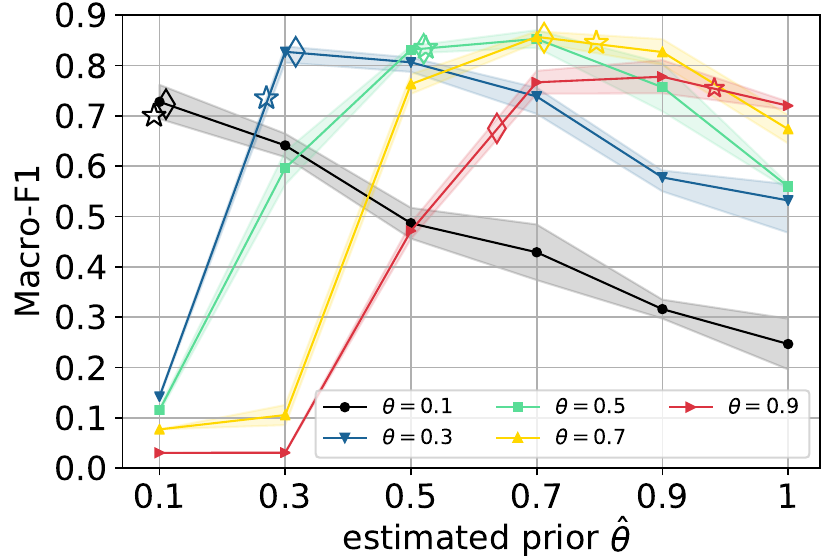}
\setcaptionwidth{5.2cm}
\vspace{-3mm}
\caption{\small{Influence and estimation accuracy of mixture proportion $\theta$.}}
\label{fig:influence} 
\end{wrapfigure}
\textbf{Influence and accuracy on the estimation of $\bf{\theta}$.}  Figure~\ref{fig:influence} plots the sensitivity curve, where the estimated prior $\hat{\theta}$ varies from 0.1 to 0.9 under different ground-truth mixture proportions $\theta$. We observe that a misspecified mixture proportion will clearly lead to performance degeneration. Interestingly, the degeneration is not isotropy---a larger misspecified value would be much more benign than a smaller one. We mark averaged estimated values of KME-base ($\blacklozenge$) and KME-shift ($\bigstar$). Evidently, the estimator gives high-quality estimated prior $\hat{\theta}$, close to the ground-truth value, which prevents our approach from performance degeneration.

\subsection{Handling Complex Changing Environments}
To answer \textsf{Q3}, we compare our approach with several baselines when augmented classes appear and prior of known classes shifts simultaneously. The experiments are simulated on mnist dataset, where classes $\{1, 3, 5, 7, 9\}$ are known and share the equal prior in labeled data. $\{2,4,6,8,10\}$ are taken as the augmented classes and account for 50$\%$ in unlabeled data. As for the prior shift in the testing distribution, we scale the the prior of five known classes to $[1-\alpha,1-\alpha/2,1,1+\alpha/2,1+\alpha]\times0.2$ respectively, where parameter $\alpha$ controls shift intensity ranging from $0$ to $0.7$.

\textbf{Contenders.} Contenders include LACU-SVM, OVR-shift and three variants of \textsc{Eulac} (\textsc{Eulac}-base, \textsc{Eulac}-base++ and \textsc{Eulac}-shift), where LACU-SVM and \textsc{Eulac}-base do not consider the shift on known classes' prior, while OVR-shift and \textsc{Eulac}-base++ take it into account but are biased. \textsc{Eulac}-shift is the unbiased estimator. For all approaches, class prior $\theta_{te}^k$ is estimated by KME-shift. Detailed descriptions of contenders can be found in Appendix~\ref{sec-appendix:experiment-3}.

\begin{wrapfigure}{r}{5.4cm}
\vspace{-4mm}
\centering
\includegraphics[width = 0.37\textwidth]{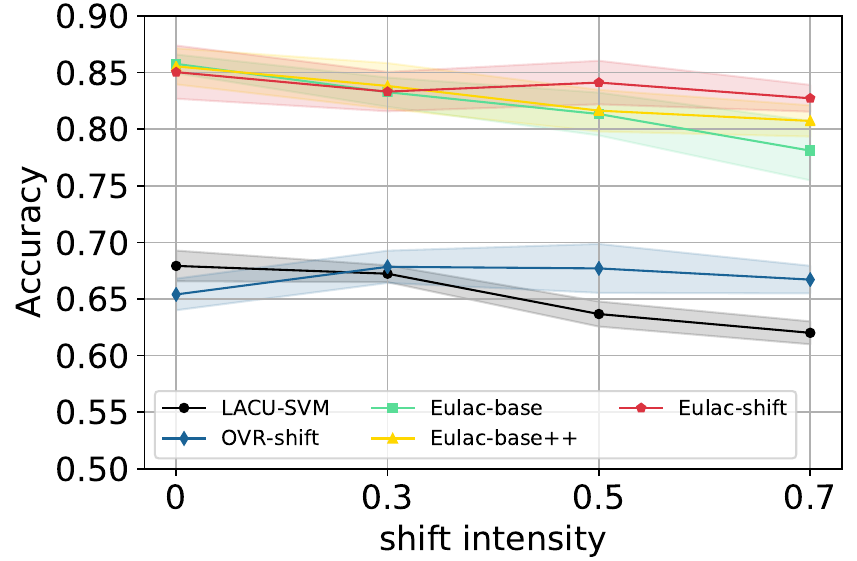}
\vspace{-3mm}
\setcaptionwidth{5.4cm}
\caption{\small{Comparison in complex environments (augmented classes \& prior shift).}}
\label{fig:target-shift}
\end{wrapfigure}

\textbf{Results.} Since Macro-F1 is an insensitive measure for the prior shift scenario, we report the accuracy for contenders in Figure~\ref{fig:target-shift}. First, with the increase of shift intensity, methods without considering prior shift (LACU-SVM, \textsc{Eulac}-base) suffer from marked performance degeneration, which shows the importance for handling distribution change of known classes with augmented classes. Besides, \textsc{Eulac}-shift achieves the best accuracy with high shift intensity and retains comparable performance with its baselines when there is no prior shift. The results validate the efficacy and safety of our proposal in complex environments. 

\section{Conclusion}
In this paper, we investigate the problem of learning with unobserved augmented classes by exploiting unlabeled training data. We introduce the \emph{class shift condition} to connect known and augmented classes, based on which an unbiased risk estimator can be established. By empirically minimizing the risk estimator with various hypothesis sets, we design the \textsc{Eulac} approach, supported by both consistency and generalization error analysis. Moreover, with the generalized class shift condition, we show the potential of our approach for handling a more general setting of class distribution change, where augmented classes appear and the prior of known classes shifts simultaneously. Extensive empirical studies confirm the effectiveness of the proposed approach. In the future, we will investigate whether our approach can be tailored for the streaming setting. Besides, it is also interesting to consider even more general scenarios of class distribution change than the problem settings studied in this paper, in order to handle more realistic changing environments. 

\newpage
\section*{Acknowledgements}
This research was supported by the NSFC (61751306, 61921006) and the Collaborative Innovation Center of Novel Software Technology and Industrialization. Meanwhile, the authors want to thank Yu-Hu Yan for reading the draft and the anonymous reviewers for the helpful and insightful comments.

\section*{Broader Impact}
In this paper, we develop the \textsc{Eulac}, an approach exploiting unlabeled data for learning with augmented classes. The augmented classes appear in many applications, such as unobserved animals appear in species recognition task~\citep{journals/aim/Dietterich17} and unexpected background images exist in object detection~\citep{Scheirer_2013_TPAMI}. Our approach offers a way to improve the robustness of the learning system for these applications by identifying the unseen augmented classes more accurately. Nevertheless, we also admit it would raise concerns when applying these techniques to some malicious applications. For example, one could employ ML systems to detect rare animals, resulting in an increased probability of rare animals being hunted and thus making the animals even more dangerous. Therefore, we should call for laws and regulations to limits the use of ML techniques in such applications.

On the other hand, it is also crucial to facilitate learning systems with the capability of tackling the augmented classes. Many applications require such robustness and will benefit from our techniques, and the potential risk is believed to be manageable with more sound human regulations.

\bibliography{new_class}

\begin{thebibliography}{48}
\providecommand{\natexlab}[1]{#1}
\providecommand{\url}[1]{\texttt{#1}}
\expandafter\ifx\csname urlstyle\endcsname\relax
  \providecommand{\doi}[1]{doi: #1}\else
  \providecommand{\doi}{doi: \begingroup \urlstyle{rm}\Url}\fi

\bibitem[Dietterich(2017)]{journals/aim/Dietterich17}
Thomas~G. Dietterich.
\newblock Steps toward robust artificial intelligence.
\newblock \emph{{AI} Magazine}, 38\penalty0 (3):\penalty0 3--24, 2017.

\bibitem[Zhou(2016)]{journals/fcsc/Zhou16a}
Zhi-Hua Zhou.
\newblock Learnware: on the future of machine learning.
\newblock \emph{Frontiers of Computer Science}, 10\penalty0 (4):\penalty0
  589--590, 2016.

\bibitem[Sugiyama and Kawanabe(2012)]{book/sugiyama/2012}
Masashi Sugiyama and Motoaki Kawanabe.
\newblock \emph{Machine Learning in Non-Stationary Environments: Introduction
  to Covariate Shift Adaptation}.
\newblock The MIT Press, 2012.

\bibitem[Gama et~al.(2014)Gama, Zliobaite, Bifet, Pechenizkiy, and
  Bouchachia]{journals/csur/GamaZBPB14}
Jo{\~{a}}o Gama, Indre Zliobaite, Albert Bifet, Mykola Pechenizkiy, and
  Abdelhamid Bouchachia.
\newblock A survey on concept drift adaptation.
\newblock \emph{{ACM} Computing Surveys}, 46\penalty0 (4):\penalty0
  44:1--44:37, 2014.

\bibitem[Zhao et~al.(2019)Zhao, Wang, Xie, Guo, and Zhou]{TKDE'19:DFOP}
Peng Zhao, Xinqiang Wang, Siyu Xie, Lei Guo, and Zhi-Hua Zhou.
\newblock Distribution-free one-pass learning.
\newblock \emph{IEEE Transaction on Knowledge and Data Engineering}, 2019.
\newblock \doi{10.1109/TKDE.2019.2937078}.

\bibitem[Da et~al.(2014)Da, Yu, and Zhou]{conf/aaai/DaYZ14}
Qing Da, Yang Yu, and Zhi-Hua Zhou.
\newblock Learning with augmented class by exploiting unlabeled data.
\newblock In \emph{Proceedings of the 28th {AAAI} Conference on Artificial
  Intelligence (AAAI)}, pages 1760--1766, 2014.

\bibitem[Sch{\"{o}}lkopf et~al.(2001)Sch{\"{o}}lkopf, Platt, Shawe{-}Taylor,
  Smola, and Williamson]{journals/neco/ScholkopfPSSW01}
Bernhard Sch{\"{o}}lkopf, John~C. Platt, John Shawe{-}Taylor, Alexander~J.
  Smola, and Robert~C. Williamson.
\newblock Estimating the support of a high-dimensional distribution.
\newblock \emph{Neural Computation}, 13\penalty0 (7):\penalty0 1443--1471,
  2001.

\bibitem[Tax and Duin(2004)]{journals/ml/TaxD04}
David M.~J. Tax and Robert P.~W. Duin.
\newblock Support vector data description.
\newblock \emph{Machine Learning}, 54\penalty0 (1):\penalty0 45--66, 2004.

\bibitem[Parzen(1962)]{parzen-1962estimation}
Emanuel Parzen.
\newblock On estimation of a probability density function and mode.
\newblock \emph{The Annals of Mathematical Statistics}, 33\penalty0
  (3):\penalty0 1065--1076, 1962.

\bibitem[Kim and Scott(2012)]{journals/jmlr/KimS12}
JooSeuk Kim and Clayton~D. Scott.
\newblock Robust kernel density estimation.
\newblock \emph{Journal of Machine Learning Research}, 13\penalty0
  (1):\penalty0 2529--2565, 2012.

\bibitem[Liu et~al.(2008)Liu, Ting, and Zhou]{conf/icdm/LiuTZ08}
Fei~Tony Liu, Kai~Ming Ting, and Zhi-Hua Zhou.
\newblock Isolation forest.
\newblock In \emph{Proceedings of the 8th {IEEE} International Conference on
  Data Mining (ICDM)}, pages 413--422, 2008.

\bibitem[Scheirer et~al.(2013)Scheirer, Rocha, Sapkota, and
  Boult]{Scheirer_2013_TPAMI}
Walter~J. Scheirer, Anderson Rocha, Archana Sapkota, and Terrance~E. Boult.
\newblock Towards open set recognition.
\newblock \emph{IEEE Transactions on Pattern Analysis and Machine
  Intelligence}, 35\penalty0 (7):\penalty0 1757–1772, 2013.

\bibitem[Scheirer et~al.(2014)Scheirer, Jain, and
  Boult]{journals/pami/ScheirerJB14}
Walter~J. Scheirer, Lalit~P. Jain, and Terrance~E. Boult.
\newblock Probability models for open set recognition.
\newblock \emph{{IEEE} Transactions on Pattern Analysis and Machine
  Intelligence}, 36\penalty0 (11):\penalty0 2317--2324, 2014.

\bibitem[Mendes{-}Junior et~al.(2017)Mendes{-}Junior, de~Souza,
  de~Oliveira~Werneck, Stein, Pazinato, de~Almeida, Penatti, da~Silva~Torres,
  and Rocha]{journals/ml/Mendes-JuniorSW17}
Pedro~Ribeiro Mendes{-}Junior, Roberto~Medeiros de~Souza, Rafael
  de~Oliveira~Werneck, Bernardo~V. Stein, Daniel~V. Pazinato, Waldir~R.
  de~Almeida, Ot{\'{a}}vio A.~B. Penatti, Ricardo da~Silva~Torres, and Anderson
  Rocha.
\newblock Nearest neighbors distance ratio open-set classifier.
\newblock \emph{Machine Learning}, 106\penalty0 (3):\penalty0 359--386, 2017.

\bibitem[Rudd et~al.(2018)Rudd, Jain, Scheirer, and
  Boult]{journals/pami/RuddJSB18}
Ethan~M. Rudd, Lalit~P. Jain, Walter~J. Scheirer, and Terrance~E. Boult.
\newblock The extreme value machine.
\newblock \emph{{IEEE} Transactions on Pattern Analysis and Machine
  Intelligence}, 40\penalty0 (3):\penalty0 762--768, 2018.

\bibitem[Scott and Blanchard(2009)]{journals/jmlr/ScottB09}
Clayton Scott and Gilles Blanchard.
\newblock Novelty detection: Unlabeled data definitely help.
\newblock In \emph{Proceedings of the 12th International Conference on
  Artificial Intelligence and Statistics (AISTATS)}, pages 464--471, 2009.

\bibitem[Blanchard et~al.(2010)Blanchard, Lee, and
  Scott]{journals/jmlr/BlanchardLS10}
Gilles Blanchard, Gyemin Lee, and Clayton Scott.
\newblock Semi-supervised novelty detection.
\newblock \emph{Journal of Machine Learning Research}, 13:\penalty0 2973--3009,
  2010.

\bibitem[Liu et~al.(2018)Liu, Garrepalli, Dietterich, Fern, and
  Hendrycks]{conf/icml/LiuGDFH18}
Si~Liu, Risheek Garrepalli, Thomas~G. Dietterich, Alan Fern, and Dan Hendrycks.
\newblock Open category detection with {PAC} guarantees.
\newblock In \emph{Proceedings of the 35th International Conference on Machine
  Learning {(ICML)}}, pages 3175--3184, 2018.

\bibitem[Zhang(2004)]{journals/jmlr/Zhang04a}
Tong Zhang.
\newblock Statistical analysis of some multi-category large margin
  classification methods.
\newblock \emph{Journal of Machine Learning Research}, 5:\penalty0 1225--1251,
  2004.

\bibitem[Rifkin and Klautau(2004)]{journals/jmlr/RifkinK03}
Ryan~M. Rifkin and Aldebaro Klautau.
\newblock In defense of one-vs-all classification.
\newblock \emph{Journal of Machine Learning Research}, 5:\penalty0 101--141,
  2004.

\bibitem[Natarajan et~al.(2013)Natarajan, Dhillon, Ravikumar, and
  Tewari]{conf/nips/NatarajanDRT13}
Nagarajan Natarajan, Inderjit~S. Dhillon, Pradeep Ravikumar, and Ambuj Tewari.
\newblock Learning with noisy labels.
\newblock In \emph{Advances in Neural Information Processing Systems 26
  (NeurIPS)}, pages 1196--1204, 2013.

\bibitem[du~Plessis et~al.(2015)du~Plessis, Niu, and
  Sugiyama]{conf/icml/PlessisNS15}
Marthinus~Christoffel du~Plessis, Gang Niu, and Masashi Sugiyama.
\newblock Convex formulation for learning from positive and unlabeled data.
\newblock In \emph{Proceedings of the 32nd International Conference on Machine
  Learning (ICML)}, pages 1386--1394, 2015.

\bibitem[Scholkopf and Smola(2001)]{book'01:Kernel}
Bernhard Scholkopf and Alexander~J. Smola.
\newblock \emph{Learning with Kernels: Support Vector Machines, Regularization,
  Optimization, and Beyond}.
\newblock The {MIT} Press, 2001.

\bibitem[Kiryo et~al.(2017)Kiryo, Niu, du~Plessis, and
  Sugiyama]{conf/nips/KiryoNPS17}
Ryuichi Kiryo, Gang Niu, Marthinus~Christoffel du~Plessis, and Masashi
  Sugiyama.
\newblock Positive-unlabeled learning with non-negative risk estimator.
\newblock In \emph{Advances in Neural Information Processing Systems 30
  (NeurIPS)}, pages 1675--1685, 2017.

\bibitem[Ramaswamy et~al.(2016)Ramaswamy, Scott, and
  Tewari]{conf/icml/RamaswamyST16}
Harish~G. Ramaswamy, Clayton Scott, and Ambuj Tewari.
\newblock Mixture proportion estimation via kernel embeddings of distributions.
\newblock In \emph{Proceedings of the 33rd International Conference on Machine
  Learning ({ICML})}, pages 2052--2060, 2016.

\bibitem[Elkan and Noto(2008)]{conf/kdd/ElkanN08}
Charles Elkan and Keith Noto.
\newblock Learning classifiers from only positive and unlabeled data.
\newblock In \emph{Proceedings of the 14th {ACM} {SIGKDD} International
  Conference on Knowledge Discovery and Data Mining (KDD)}, pages 213--220,
  2008.

\bibitem[Jain et~al.(2016)Jain, White, Trosset, and
  Radivojac]{journals/corr/JainWTR16}
Shantanu Jain, Martha White, Michael~W. Trosset, and Predrag Radivojac.
\newblock Nonparametric semi-supervised learning of class proportions.
\newblock arXiv:1601.01944, 2016.

\bibitem[du~Plessis et~al.(2017)du~Plessis, Niu, and
  Sugiyama]{journals/ml/PlessisNS17}
Marthinus~Christoffel du~Plessis, Gang Niu, and Masashi Sugiyama.
\newblock Class-prior estimation for learning from positive and unlabeled data.
\newblock \emph{Machine Learning}, 106\penalty0 (4):\penalty0 463--492, 2017.

\bibitem[Bousquet et~al.(2003)Bousquet, Boucheron, and Lugosi]{MLSS'03:STL}
Olivier Bousquet, St{\'{e}}phane Boucheron, and G{\'{a}}bor Lugosi.
\newblock Introduction to statistical learning theory.
\newblock In \emph{Advanced Lectures on Machine Learning, Machine Learning
  Summer Schools 2003}, pages 169--207, 2003.

\bibitem[Mohri et~al.(2012)Mohri, Rostamizadeh, and
  Talwalkar]{book'12:foundation}
Mehryar Mohri, Afshin Rostamizadeh, and Ameet Talwalkar.
\newblock \emph{Foundations of Machine Learning, 2nd Edition}.
\newblock The {MIT} Press, 2012.

\bibitem[Zhou and Chen(2002)]{journals/kbs/ZhouC02}
Zhi-Hua Zhou and Zhaoqian Chen.
\newblock Hybrid decision tree.
\newblock \emph{Knowledge-Based Systems}, 15\penalty0 (8):\penalty0 515--528,
  2002.

\bibitem[Ding et~al.(2018)Ding, Liu, and Zhang]{conf/icdm/DingLZ18}
Si{-}Yu Ding, Xu{-}Ying Liu, and Min{-}Ling Zhang.
\newblock Imbalanced augmented class learning with unlabeled data by label
  confidence propagation.
\newblock In \emph{Proceedings of the 18th {IEEE} International Conference on
  Data Mining (ICDM)}, pages 79--88, 2018.

\bibitem[Fink et~al.(2006)Fink, Shalev{-}Shwartz, Singer, and
  Ullman]{conf/icml/FinkSSU06}
Michael Fink, Shai Shalev{-}Shwartz, Yoram Singer, and Shimon Ullman.
\newblock Online multiclass learning by interclass hypothesis sharing.
\newblock In \emph{Proceedings of the 23rd International Conference on Machine
  Learning ({ICML})}, pages 313--320, 2006.

\bibitem[Muhlbaier et~al.(2009)Muhlbaier, Topalis, and
  Polikar]{journals/tnn/MuhlbaierTP09}
Michael~D. Muhlbaier, Apostolos Topalis, and Robi Polikar.
\newblock Learn\({}^{\mbox{++}}\).nc: Combining ensemble of classifiers with
  dynamically weighted consult-and-vote for efficient incremental learning of
  new classes.
\newblock \emph{IEEE Transactions on Neural Networks and Learning Systems},
  20\penalty0 (1):\penalty0 152--168, 2009.

\bibitem[Mu et~al.(2017)Mu, Ting, and Zhou]{journals/tkde/MuTZ17}
Xin Mu, Kai~Ming Ting, and Zhi-Hua Zhou.
\newblock Classification under streaming emerging new classes: {A} solution
  using completely-random trees.
\newblock \emph{IEEE Transactions on Knowledge and Data Engineering},
  29\penalty0 (8):\penalty0 1605--1618, 2017.

\bibitem[Cai et~al.(2019)Cai, Zhao, Ting, Mu, and Jiang]{ICDM'19:SENNE}
Xin-Qiang Cai, Peng Zhao, Kai~Ming Ting, Xin Mu, and Yuan Jiang.
\newblock Nearest neighbor ensembles: An effective method for difficult
  problems in streaming classification with emerging new classes.
\newblock In \emph{Proceedings of the 19th International Conference on Data
  Mining (ICDM)}, pages 970--975, 2019.

\bibitem[Geng et~al.(2018)Geng, Huang, and Chen]{arxiv:OSR-survey}
Chuanxing Geng, Sheng{-}Jun Huang, and Songcan Chen.
\newblock Recent advances in open set recognition: {A} survey.
\newblock arXiv: 1811.08581, 2018.

\bibitem[Yu et~al.(2017)Yu, Qu, Li, and Guo]{conf/ijcai/YuQLG17}
Yang Yu, Wei-Yang Qu, Nan Li, and Zimin Guo.
\newblock Open category classification by adversarial sample generation.
\newblock In \emph{Proceedings of the Twenty-Sixth International Joint
  Conference on Artificial Intelligence (IJCAI)}, pages 3357--3363, 2017.

\bibitem[Liu et~al.(2002)Liu, Lee, Yu, and Li]{conf/icml/LiuLYL02}
Bing Liu, Wee~Sun Lee, Philip~S. Yu, and Xiaoli Li.
\newblock Partially supervised classification of text documents.
\newblock In \emph{Proceedings of the Nineteenth International Conference on
  Machine Learning (ICML)}, pages 387--394, 2002.

\bibitem[Ward et~al.(2009)Ward, Hastie, Barry, Elith, and
  Leathwick]{journals/biometrics/Ward09}
Gill Ward, Trevor Hastie, Simon Barry, Jane Elith, and John~R Leathwick.
\newblock Presence-only data and the {EM} algorithm.
\newblock \emph{Biometrics}, 65\penalty0 (2):\penalty0 554--563, 2009.

\bibitem[du~Plessis et~al.(2014)du~Plessis, Niu, and
  Sugiyama]{conf/nips/PlessisNS14}
Marthinus~Christoffel du~Plessis, Gang Niu, and Masashi Sugiyama.
\newblock Analysis of learning from positive and unlabeled data.
\newblock In \emph{Advances in Neural Information Processing Systems 27
  ({NeurIPS})}, pages 703--711, 2014.

\bibitem[Xu et~al.(2017)Xu, Xu, Xu, and Tao]{conf/ijcai/XuX0T17}
Yixing Xu, Chang Xu, Chao Xu, and Dacheng Tao.
\newblock Multi-positive and unlabeled learning.
\newblock In \emph{Proceedings of the 26th International Joint Conference on
  Artificial Intelligence (IJCAI)}, pages 3182--3188, 2017.

\bibitem[Tsuchiya et~al.(2019)Tsuchiya, Charoenphakdee, Sato, and
  Sugiyama]{journals/corr/abs-1901-11351}
Taira Tsuchiya, Nontawat Charoenphakdee, Issei Sato, and Masashi Sugiyama.
\newblock Semi-supervised ordinal regression based on empirical risk
  minimization.
\newblock arXiv:1901.11351, 2019.

\bibitem[Bendale and Boult(2016)]{conf/cvpr/BendaleB16}
Abhijit Bendale and Terrance~E. Boult.
\newblock Towards open set deep networks.
\newblock In \emph{Proceedings of the {IEEE} Conference on Computer Vision and
  Pattern Recognition (CVPR)}, pages 1563--1572, 2016.

\bibitem[Ge et~al.(2017)Ge, Demyanov, and Garnavi]{conf/bmvc/GeDG17}
Zongyuan Ge, Sergey Demyanov, and Rahil Garnavi.
\newblock Generative openmax for multi-class open set classification.
\newblock In \emph{Proceedings of the British Machine Vision Conference 2017
  (BMCV)}, 2017.

\bibitem[Neal et~al.(2018)Neal, Olson, Fern, Wong, and
  Li]{conf/eccv/NealOFWL18}
Lawrence Neal, Matthew Olson, Xiaoli Fern, Weng-Keen Wong, and Fuxin Li.
\newblock Open set learning with counterfactual images.
\newblock In \emph{Proceedings of the European Conference on Computer Vision
  (ECCV)}, pages 613--628, 2018.

\bibitem[Scott(2015)]{conf/aistats/Scott15}
Clayton Scott.
\newblock A rate of convergence for mixture proportion estimation, with
  application to learning from noisy labels.
\newblock In \emph{Proceedings of the 18th International Conference on
  Artificial Intelligence and Statistics (AISTATS)}, pages 838--846, 2015.

\bibitem[Koltchinskii(2011)]{book/springer/oracle2011}
Vladimir Koltchinskii.
\newblock \emph{Oracle Inequalities in Empirical Risk Minimization and Sparse
  Recovery Problems}.
\newblock Springer, 2011.

\end{thebibliography}
\bibliographystyle{unsrtnat}

\clearpage
\appendix
\begin{center}
  {\bf \LARGE Supplementary Material for ``An Unbiased Risk Estimator for Learning with Augmented Classes''}
\end{center}

\appendix

This is the supplemental material for the paper ``An Unbiased Risk Estimator for Learning with Augmented Classes''. The appendix is organized as follows.
\begin{itemize}
\item Appendix~\ref{sec-appendix:algorithm}: supplementary descriptions for our approach, including the non-negative LAC risk for the deep model training and the introduction to the mixture proportion estimation.
\begin{itemize}
	\item Appendix~\ref{sec-appendix:deep}: the non-negative risk used for the deep model training.
	\item Appendix~\ref{sec-appendix:MPE}: introduction for the mixture proportion estimation (MPE) problem, including the problem formulation, the KME-based approach with its theoretical analysis and implementation with dimensionality reduction.
\end{itemize}

\item Appendix~\ref{sec-appendix:proof-all}: proofs including the technical lemmas for the results in main paper.
\begin{itemize}
  \item Appendix~\ref{sec-appendix:proof-prop1}: proof of Proposition~\ref{prop:LACU-equal-nonconvex}, the equality of the LAC risk and the OVR risk over testing distribution.
  \item Appendix~\ref{sec-appendix:proof-ISC}: proofs of Theorem~\ref{thm:consistent} and Theorem~\ref{thm:consistent-square loss}, the infinite-sample consistency of the \textsc{Eulac} approach.
  \item Appendix~\ref{sec-appendix:proof-FSC}: proofs of Theorem~\ref{thm:generalization-bound} and Theorem~\ref{thm:estimation-error}, the finite-sample convergence of the \textsc{Eulac} approach with the kernel-based hypothesis space.
\end{itemize}
\item Appendix~\ref{sec-appendix:experiments}: more empirical results and detailed settings for experiments.
\begin{itemize}
  \item Appendix~\ref{sec-appendix:experiments-1}: detailed descriptions for the comparison on RKHS-based \textsc{Eulac}.
  \item Appendix~\ref{sec-appendix:experiments-2}: detailed descriptions for the comparison on DNN-based \textsc{Eulac}.
  \item Appendix~\ref{sec-appendix:experiment-3}: detailed descriptions for the comparison in the complex changing environments.
\end{itemize}
\end{itemize}

\section{Supplementary Descriptions for \textsc{Eulac}}
\label{sec-appendix:algorithm}
This section introduces the non-negative LAC risk for training the deep model and the mixture proportion estimation problem. 
\subsection{Rectified Non-negative LAC Risk for Deep Model}
\label{sec-appendix:deep}
We first show the importance of the non-negative risk for training the deep model. As presented in Proposition~\ref{prop:LACU-equal-nonconvex}, classifiers' risk over the testing distribution with respect to the OVR loss
\begin{align*}
R_{\psi} ={}& \mathbb{E}_{(\mathbf{x},y)\sim P_{te}}\left[\psi(f_y(\mathbf{x}))+\sum_{k=1,k\neq y}^{K+1}\psi(-f_k(\mathbf{x}))\right]
\end{align*}
can be assessed with distributions of labeled and unlabeled data by
\begin{equation*}
\begin{split}
R_{LAC} = {}& \theta\cdot\mathbb{E}_{(\mathbf{x},y)\sim P_\textsf{kc}}\big[\psi(f_y(\mathbf{x})) + \psi(-f_{\textsf{ac}}(\mathbf{x}))\big]\\
&\ + \underbrace{\mathbb{E}_{\mathbf{x}\sim p^{te}_X(\x)}\left[\psi(f_{\textsf{ac}}(\mathbf{x}))+\sum_{k=1}^K\psi(-f_{k}(\mathbf{x}))\right] - \theta\cdot\mathbb{E}_{(\mathbf{x},y)\sim P_\textsf{kc}}\left[\psi(-f_y(\mathbf{x})) +\psi(f_{\textsf{ac}}(\mathbf{x}))\right]}_{\coloneqq R_{LAC}^{\textsf{ac}}},
\end{split}
\end{equation*}
Although the expected value of $R_{LAC}$ and $R_{\psi}$ are equivalent, their empirical formulations could behave differently. The empirical version $\hat{R}_{\psi}$ is always bounded from below by 0 since the loss function $\psi$ is non-negative. However, the empirical LAC risk $\hat{R}_{LAC}$ could go negative since the term $\hat{R}_{LAC}^{\textsf{ac}}$ is not guaranteed to bound from below. As observed by~\citet{conf/nips/KiryoNPS17}, the negative empirical risk would lead to severe over-fitting when complex models such as deep neural networks are employed. Thus, we require to rectify the empirical LAC risk.

To avoid the undesired phenomenon, we extend the non-negative risk~\citep{conf/nips/KiryoNPS17} for the OVR scheme. Since the negative terms only come from $\hat{R}^{\textsf{ac}}_{LAC}$, we only need to rectify the corresponding part of the empirical LAC risk $\hat{R}_{LAC}$. Under such a case, the \emph{non-negative} empirical LAC risk is written as
\begin{align}
 \tilde{R}_{LAC} ={}&\frac{\theta}{n_l}\sum_{i=1}^{n_l}\big[\psi(-f_{y_i}(\mathbf{x}_i)) +\psi(f_{\textsf{ac}}(\mathbf{x}_i))\big] + \max\Big\{0, \frac{1}{n_u}\sum_{i=1}^{n_u}\psi(f_\textsf{ac}(\x_i))-\frac{\theta}{n_l}\sum_{i=1}^{n_l}\psi(f_\textsf{ac}(\x_i))\Big\}\notag\\
 &\quad+\sum_{k=1}^K \max\Big\{0, \frac{1}{n_u}\sum_{i=1}^{n_u}\psi(-f_k(\x_i))-\frac{\theta}{n_l}\sum_{i=1}^{n_l}\psi(-f_k(\x_i))\Big\},\label{eq:nnOVR}
\end{align}
where we add the maximum operator for each binary classifier to avoid the negative loss. The non-negative empirical LAC risk $\tilde{R}_{LAC}$ can be optimized with the gradient descent algorithm in~\citep{conf/nips/KiryoNPS17}.

\subsection{Mixture Proportion Estimation}
\label{sec-appendix:MPE}
In this section, we first introduce the formulation of the MPE problem, followed by the description for the KME-based estimator~\citep{conf/icml/RamaswamyST16} with its theoretical analysis and practical implementation.
\paragraph{Problem formulation.} Let $G$, $H$ be distributions over a compact metric space $\mathcal{X}$ with supports given by $supp(G)$ and $supp(H)$. Let $\theta\in[0,1)$ and $F$ be a distribution that is a mixture of $G$ and $H$ with the proportion $\theta$,
\begin{equation}
F = (1-\theta)\cdot G+\theta\cdot H.
\end{equation}
The object of MPE is to estimate $\theta$ with the empirical observations $\hat{H}=\{\x_i\}_{i=1}^{n_h}$ sampled i.i.d. from $H$ and $\hat{F}=\{\x_i\}_{i=n_h+1}^{n_h+n_f}$ sampled i.i.d. from $F$.

Without any assumption, the true mixture proportion is not identifiable in general, which makes MPE an ill-defined problem. To understand this, we could consider a quick case, where $H$ and $G$ share the same distribution. In such a case, the true mixture proportion can be any value in $[0,1]$. In literature, the irreducibility assumption~\citep{journals/jmlr/BlanchardLS10} and its variants~\citep{conf/aistats/Scott15,conf/icml/RamaswamyST16}, which essentially state that $G$ cannot be expressed as any non-trivial mixture of $H$ and some other distributions, are introduced to ensure a unique and thus a identifiable mixture proportion. 

\paragraph{KME-based approach.} In this paper, we employ the KME-based estimator proposed by \citet{conf/icml/RamaswamyST16} for solving the MPE problem since it enjoys a clear theoretical understanding with nice empirical performance. Under the irreducibility assumption, authors' key observation is that the true mixture proportion $\theta$ is exactly the maximum value $\hat{\theta}$ making $G' = (F-\hat{\theta}\cdot H)/(1-\hat{\theta})$ a legal distribution. More specifically, with the growth of $\hat{\theta}$, the distribution $G'$ changes from a valid distribution to an illegal done, whose density function could be negative. Thus, if we can judge whether $G'$ is a valid distribution, then the binary search can be used for identifying the critical value.

Based on the observation, they develop an approach based on kernel mean embedding (KME), where the distribution $H$ and $F$ are embedded into a RKHS $\mathbb{H}$ as $\Phi_H = \mathbb{E}_{\x\sim H}[\phi(\x)]$ and $\Phi_G= \mathbb{E}_{\x\sim G}[\phi(\x)]$. Function $\phi:\mathcal{X}\mapsto\mathbb{H}$ is the mapping associated with kernel $\kappa:\mathcal{X}\times \X\mapsto \mathbb{R}$. To judge whether $G'$ is a legal distribution, they specify a set including embeddings of all legal distributions $\mathcal{C}=\{w \in \mathbb{H}: w=\Phi_P, \text { for some distributions } P\}$, and then calculate the distance
\begin{equation*}
d(\hat{\theta}) =\inf _{w \in \mathcal{C}}\|\Phi_{G'}-w\|_{\mathbb{H}} =\inf _{w \in \mathcal{C}}\| (\Phi_F-\hat{\theta}\cdot \Phi_H)/(1-\hat{\theta})-w\|_{\mathbb{H}}.
\end{equation*}
 We can see that when $\hat{\theta}\leq\theta$, $d(\hat{\theta})=0$, and $\hat{\theta}>\theta$, $d(\hat{\theta})>0$. As a result, binary search can be used for finding the true mixture proportion which is essentially the maximum value of $\hat{\theta}$ making $G'$ a legal distribution, i.e., $d(\hat{\theta})=0$.

\citet{conf/icml/RamaswamyST16} search for $\lambda = 1/(1-\theta)$ instead of $\theta$ to identify the mixture proportion, where the distance function is,
$$d(\hat{\lambda})=\inf _{w \in \mathcal{C}}\|\hat{\lambda} \Phi_F+(1-\hat{\lambda}) \Phi_H-w\|_{\mathbb{H}}.$$
Similarly, we have $d(\hat{\lambda})=0$ for all $\hat{\lambda}\in[0,\lambda]$ and $d(\hat{\lambda})>0$ for all $\hat{\lambda}\in[\lambda,+\infty)$. Besides the value of $d(\hat{\lambda})$, its gradient $\nabla d(\hat{\lambda})$ is also informative for identifying the true mixture proportion, as $\nabla d(\hat{\lambda}) = 0$ for all $\hat{\lambda}\in[0,\lambda]$ and $\nabla d(\hat{\lambda}) > 0$ for all $\hat{\lambda}\in[\lambda,+\infty)$. In this paper, we use the estimator based on the gradient, as it achieves more stable empirical performance as shown in~\citet{conf/icml/RamaswamyST16}.

However, distributions $H$ and $F$ are unaccessible, where only their observations $\hat{H}$ and $\hat{F}$ are available. Thus, in the practical implementation, we can empirically approximate $d(\hat{\lambda})$ by
$$\widehat{d}(\hat{\lambda})=\inf _{w \in \mathcal{C}_{S}}\|\hat{\lambda} \Phi_{\widehat{F}}+(1-\hat{\lambda}) \Phi_{\hat{H}}-w\|_{\mathbb{H}},$$
where $\mathcal{C}_{S}=\left\{w \in \mathbb{H}: w=\sum_{i=1}^{n_h+n_f} \alpha_{i} \phi\left(x_{i}\right), \text { for some } \alpha \in \Delta_{n_h+n_f}\right\}$ and $\Delta_{n_h+n_f}\subseteq \mathbb{R}^{n_h+n_f}$ is the probability simplex. After obtaining $\widehat{d}(\hat{\lambda})$, the value of $\nabla\widehat{d}(\hat{\lambda})$ can be calculated as $\nabla\widehat{d}(\hat{\lambda}) = (\widehat{d}(\hat{\lambda}+\epsilon/2)-\widehat{d}(\hat{\lambda}-\epsilon/2))/\epsilon$. However, the empirical gradient $\nabla \hat{d}(\cdot)$ does not enjoy the nice property as $\nabla d(\cdot)$, since $\lambda$ is no longer a critical value. We should specify a threshold $\nu$ instead of $0$ for identifying the true mixture proportion. The whole algorithm is summarized in Algorithm~\ref{alg:KME}.

\begin{algorithm}[!t]
\caption{KME gradient estimator}
\label{alg:KME}
\begin{multicols}{2}
  \begin{algorithmic}[1]
  \STATE \textbf{Input:} $\hat{H}=\{\x_i\}_{i=1}^{n_h}$ drawn from distribution $H$ and $\hat{F}=\{\x_i\}_{i=n_h}^{n_h+n_f}$ drawn from distribution $F$, kernel $k:\X\times\X\mapsto[0,+\infty)$, threshold $\nu\in[0,+\infty)$
  \STATE \textbf{Constants:} $\epsilon=0.02$, $\lambda_{\mathrm{UB}} = 10$
  \STATE \textbf{Output:} $\hat{\lambda}_{\nu}^G$
  \STATE $\hat{\lambda}_{\text {left }}=1, \hat{\lambda}_{\text {right }}=\hat{\lambda}_{\mathrm{UB}}$
  \STATE $K_{i, j}=\kappa\left(\mathrm{x}_{i}, \mathrm{x}_{j}\right)$ for $1 \leq i, j \leq n_h+n_f$
  \WHILE {$\hat{\lambda}_{\text {right }}- \hat{\lambda}_{\text {left }}>\epsilon$}
  \STATE $\hat{\lambda}_{\mathrm{curr}}=(\hat{\lambda}_{\mathrm{right}}+\hat{\lambda}_{\mathrm{left}})/2$
  \STATE calculate $d_1 = \hat{d}(\hat{\lambda}_{\mathrm{curr}}+\epsilon/2)$
  \STATE calculate $d_2 = \hat{d}(\hat{\lambda}_{\mathrm{curr}}-\epsilon/2)$
  \STATE calculate $s = \frac{(d_1-d_2)}{\epsilon}$
  \IF {$s>\nu$}
  \STATE $\hat{\lambda}_{\text {right }} = \hat{\lambda}_{\text {curr }}$
  \ELSE
    \STATE $\hat{\lambda}_{\text {left }} = \hat{\lambda}_{\text {curr }}$
  \ENDIF
  \ENDWHILE
  \STATE{Return:} $\hat{\lambda}_{\nu}^G = \hat{\lambda}_{\text{curr}}$
  \end{algorithmic}
\end{multicols}
\end{algorithm}

\paragraph{Theoretical guarantees.}
~\citet{conf/icml/RamaswamyST16} show that under the separability condition, the estimator $\hat{\lambda}_{\nu}^G$ returned by the KME-based algorithm converges to the true underlying value $\lambda$ as Theorem~\ref{thm:KME-converge}. For self-contentedness we present the assumption and theorem here.

\begin{myDef}[Definition 9 of~\citet{conf/icml/RamaswamyST16}]
A class of functions $\mathbb{H}$, and and distributions $G$, $H$ are said to satisfy the \emph{separability condition} with the margin $\alpha$ > 0 and the tolerance $\beta$, if there exists a function $h\in \mathbb{H}$ such that $\norm{h}_{\mathbb{H}}<1$ and
$$\E_{\x \sim G} [h(\x)] \leq \inf_{\x} h(\x)+\beta \leq \E_{\x \sim H} [h(\x)]-\alpha.$$
Moreover, we say that a kernel $\kappa$ and distributions $G$, $H$ satisfy the separability condition, if the unit norm ball in its RKHS and distributions $G$, $H$ satisfy the separability condition.
\end{myDef}
Based on the separability condition, we have the following convergence guarantee of Algorithm~\ref{alg:KME}.
\begin{myThm}[{Theorem 13 of~\citet{conf/icml/RamaswamyST16}}]
\label{thm:KME-converge}
Let $\kappa(\x,\x)\leq 1$ for all $\x\in\X$. Let the kernel $\kappa$, and distribution $G$, $H$ satisfy the separability condition with tolerance $\beta$ and margin $\alpha>0$. Let $\nu \in\left[\frac{\alpha}{4 \lambda}, \frac{3 \alpha}{4 \lambda}\right]$ and $\sqrt{\min (n_h, n_f)} \geq \frac{36 \sqrt{\log (1 / \delta)}}{\frac{\alpha}{\lambda}-\nu}$. Then with probability $1-4\delta$, we have
\begin{equation*}\begin{aligned}
\lambda^-\widehat{\lambda}_{\nu}^{G} &\leq c \cdot \sqrt{\log (1 / \delta)} \cdot(\min (n_h, n_f))^{-1 / 2} \\
\widehat{\lambda}_{\nu}^{G}-\lambda &\leq \frac{4 \beta \lambda^{*}}{\alpha}+c^{\prime} \cdot \sqrt{\log (1 / \delta)} \cdot(\min (n_h, n_f))^{-1 / 2}
\end{aligned}
\end{equation*}
for constants $c=\left(2 \lambda-1+\sqrt{2 \lambda^{*}}\right) \cdot \frac{12 \lambda}{\alpha}$ and $c' = \frac{144\left(\lambda\right)^{2}(\alpha+4 \beta)}{\alpha^{2}}$.
\end{myThm}
Theorem~\ref{thm:KME-converge} shows that under the mild condition, the estimator $\hat{\lambda}_{\nu}^G$ converges to $\lambda$ with high probability, with the increase of the samples drawn from $H$ and $F$. In our problem, $H$ and $F$ are the marginal distributions of labeled and unlabeled data, respectively. Consequently, we can identify the true mixture proportion with sound guarantees. 

\paragraph{Dimensionality reduction.} In practice, the KME-based approach could suffer from the curse of dimensionality. Inspired by the recent work~\citep{journals/corr/JainWTR16}, we use a pre-trained model to reduce the dimensionality of original input to its probability outputs, where the model is trained by taking the labeled data as positive and the unlabeled data as negative. For the RKHS-based \textsc{Eulac}, the probability outputs are the decision values of the OVR classifiers after calibration. For the deep model, outputs of the softmax layer can be employed.

In experiments, we find that the KME-based estimator with dimensionality reduction enjoys a more accurate estimation than that without dimensionality reduction. For the RKHS-based \textsc{Eulac}, as presented by Figure~\ref{fig:influence}, a more accurate mixture proportion $\theta$ leads to better performance. However, it is surprising that deep models show a different situation. Table~\ref{table:auc-loss} in Appendix~\ref{sec-appendix:experiments-2} illustrates that a relative larger $\hat{\theta}$ can help networks converge better, which finally leads to better AUC. Thus, we select $\theta$ as the one achieving the minimum training loss rather than estimating it.

\section{Proofs}
\label{sec-appendix:proof-all}
In this section, we present the proof of results in the main paper.

\subsection{Proof of Proposition~\ref{prop:LACU-equal-nonconvex}}
\label{sec-appendix:proof-prop1}
\begin{proof}
For simplicity, we substitute $f_{\textsf{ac}}$ by $f_{K+1}$ in the proof. First, recall that the risk of OVR strategy is defined as,
\[
 R_\psi(f_1,\dots,f_K,f_{K+1}) = \mathbb{E}_{(\mathbf{x},y)\sim P_{te}}\left[\psi(f_y(\mathbf{x}))+\sum_{k=1,k\neq y}^{K+1}\psi(-f_k(\mathbf{x}))\right].
 \]
 According to the class shift condition, we have

 \begin{equation}
     \label{eq:proof-prop1-1}
  \begin{aligned}
R_\psi(f_1,\dots,f_K,f_{\textsf{ac}}) &= \mathbb{E}_{(\mathbf{x},y)\sim P_{te}}\left[\psi(f_y(\mathbf{x}))+\sum_{k=1,k\neq y}^{K+1}\psi(-f_k(\mathbf{x}))\right]\\
& = \underbrace{\theta\cdot\mathbb{E}_{(\mathbf{x},y)\sim P_\textsf{kc}}\left[\psi(f_y(\mathbf{x}))+\sum_{k=1,k\neq y}^{K+1}\psi(-f_k(\mathbf{x}))\right]}_{\mathtt{term~A}}\\
&\qquad + \underbrace{(1-\theta)\cdot \mathbb{E}_{(\mathbf{x},y)\sim P_{\textsf{ac}}}\left[\psi(f_y(\mathbf{x}))+\sum_{k=1,k\neq y}^{K+1}\psi(-f_k(\mathbf{x}))\right]}_{\mathtt{term~B}}.
  \end{aligned}
 \end{equation}

Since $p_{XY}^{\textsf{ac}}(\x,y) = 0$ holds for all $\x\in\X$ and $y\neq \textsf{ac}$, we can reform the term $B$ into,
\begin{equation*}
    \begin{aligned}
        \mathtt{term~B} = (1-\theta)\cdot \mathbb{E}_{\mathbf{x}\sim p^{\textsf{ac}}_X(\x)}\left[\psi(f_{K+1}(\mathbf{x}))+\sum_{k=1}^{K}\psi(-f_k(\mathbf{x}))\right].
    \end{aligned}
\end{equation*}
The marginal distribution of augmented classes $p^{\textsf{ac}}_X(\x)$ is unknown, we reduce it to the difference of the marginal distributions of labeled and unlabeled data. Under the class shift condition, we have,
\begin{equation*}
p^{te}_{XY}(\mathbf{x},y) = \theta\cdot p^\textsf{kc}_{XY}(\mathbf{x},y) + (1-\theta)\cdot p^{\textsf{ac}}_{XY}(\mathbf{x},y).
\end{equation*}

By summing over the label space $\Y$, the marginal distribution $p_{X}^{\textsf{ac}}(\x)$ is obtained as,
    \begin{equation*}
        (1-\theta)\cdot p^{\textsf{ac}}_X(\mathbf{x}) =  p^{te}_X(\mathbf{x})-\theta\cdot p^\textsf{kc}_X(\mathbf{x}),
    \end{equation*}
    
    where the term $B$ can be further converted to the following form,
    \normalsize
    
    \begin{align}
        \mathtt{term~B} &= (1-\theta)\cdot\mathbb{E}_{\mathbf{x}\sim p^{\textsf{ac}}_X(\x)}\left[\psi(f_{K+1}(\mathbf{x}))+\sum_{k=1}^{K}\psi(-f_k(\mathbf{x}))\right]\notag\\
        & = \mathbb{E}_{\mathbf{x}\sim p^{te}_X(\x)}\left[\psi(f_{K+1}(\mathbf{x}))+\sum_{k=1}^{K}\psi(-f_k(\mathbf{x}))\right]\notag\\
        & \quad -\theta\cdot\mathbb{E}_{\mathbf{x}\sim p^\textsf{kc}_X(\x)}\left[\psi(f_{K+1}(\mathbf{x}))+\sum_{k=1}^{K}\psi(-f_k(\mathbf{x}))\right].\label{eq:proof-prop1-2}
    \end{align}
Plugging ~\eqref{eq:proof-prop1-2} into~\eqref{eq:proof-prop1-1}, we complete the proof.
\end{proof}

\subsection{Proofs of Theorem~\ref{thm:consistent} and Theorem~\ref{thm:consistent-square loss}}
\label{sec-appendix:proof-ISC}
Before showing the proofs of Theorem~\ref{thm:consistent} and Theorem~\ref{thm:consistent-square loss}, for self-contentedness, we introduce results regrading infinite-sample consistency (ISC) of OVR strategy provided by~\citet{journals/jmlr/Zhang04a}. 

First, we present the ISC property of the OVR strategy, which is stated as follows.
\begin{myThm}[Theorem 10 of~\citet{journals/jmlr/Zhang04a}]
\label{lemma-appen:ISC-OVR}
    Consider the OVR method, whose surrogate loss function is defined as $\Psi_y(\bm{f}) = \psi(f_y)+\sum_{k=1,k\neq y}^K\psi(-f_k)$. Assume $\psi$ is convex, bounded below, differentiable, and $\psi(z)<\psi(-z)$ when $z>0$. Then, OVR method is ISC on $\Omega = \mathbb{R}^K$ with respect to 0-1 classification risk. 
\end{myThm}

Then, we show the relationship between the risk of an ISC method and the Bayes error as the following theorem,
\begin{myThm}[Theorem 3 of~\citet{journals/jmlr/Zhang04a}]
    \label{lemma-appen:ISC-Bayes_error}
    Let $\mathcal{B}$ be the set of all vector Borel measurable functions, which take values in $\mathbb{R}^K$. For $\Omega\subset\mathbb{R}^K$, let $\mathcal{B}_{\Omega} = \{\bm{f}\in\mathcal{B}:\forall\x,\bm{f}(\x)\in\Omega\}$. If $[\Psi_y(\cdot)]$ is ISC on $\Omega$ with respect to 0-1 classification risk, then for any  $\epsilon_1>0$, there exists $\epsilon_2>0$ such that for all underlying Borel probability measurable $D$, and $\bm{f}(\cdot)\in\mathcal{B}_{\Omega}$,
    \begin{equation*}
        \mathbb{E}_{(\x,y)\sim D}[\Psi_y(\bm{f}(\x))]\leq\inf_{\bm{f}'\in\mathcal{B}_{\Omega}}\mathbb{E}_{(\x,y)\sim D}[\Psi_y(\bm{f}'(\x))] + \epsilon_2
    \end{equation*}
    implies
\begin{equation*}
    R(T(\bm{f}(\cdot)))\leq R^* + \epsilon_1,
\end{equation*}
where $T(\cdot)$ is defined as $T(\bm{f}(\x)):=\argmax_{k=1,\dots,K,\textsf{ac}}f_k(\x)$, and $R^*$ is the optimal Bayes error.
\end{myThm}

For the OVR strategy, we can further obtain a more quantitative bound.
\begin{myThm}[Theorem 11 of~\citet{journals/jmlr/Zhang04a}]
\label{lemma-appen:ISC-OVR-square_hinge}
Under the assumptions of Theorem~\ref{lemma-appen:ISC-OVR}. The function $V_{\psi}(q) = \inf_{z\in\mathbb{R}}[q\psi(z)+(1-q)\psi(-z)]$ is concave on $[0,1]$. Assume that there exists a constant $c_{\psi}>0$ such that
\[
    (q-q')^2\leq c_{\psi}^2\left(2V_{\psi}(\frac{q+q'}{2})-V_{\psi}(q)-V_{\psi}(q')\right),
\]
then we have for any $\bm{f}(\cdot)$,
\[
     R(T(\bm{f}(\cdot))) -  R^*\leq c_{\psi}\left(\mathbb{E}_{(\x,y)\sim D}[\Psi_y(\bm{f}(\x))]-\inf_{\bm{f}'\in\mathcal{B}_{\Omega}}\mathbb{E}_{(\x,y)\sim D}[\Psi_y(\bm{f}'(\x))]\right)^{1/2}.
\]
    
\end{myThm}

The proofs of Theorem~\ref{thm:consistent} and Theorem~\ref{thm:consistent-square loss} are direct consequences of Proposition~\ref{prop:LACU-equal-nonconvex} and aforementioned Theorems, which are stated as follows.
\begin{proof}[Proof of Theorem~\ref{thm:consistent}]
According to Proposition~\ref{prop:LACU-equal-nonconvex}, the risk $R_{LAC}$ equals to the risk of OVR strategy $R_\psi = \mathbb{E}_{(\x,y)\sim P_{te}}[\Psi_y(\bm{f}(\x))]$. Therefore, to prove the infinity-sample consistency of $R_{LAC}$, it is sufficient to demonstrate such a property of OVR strategy over the testing distribution $P_{te}$, which is shown as Theorem~\ref{lemma-appen:ISC-OVR} and Theorem~\ref{lemma-appen:ISC-Bayes_error}.
\end{proof}

\begin{proof}[Proof of Theorem~\ref{thm:consistent-square loss}]
    To prove Theorem~\ref{thm:consistent-square loss}, we first show the consistency of OVR strategy with square loss. It is easy to verify that, when taking $\psi(z) = (1-z)^2/4$, we have $V_{\psi}(q) = \inf_{z\in\mathbb{R}}[q\psi(z)+(1-q)\psi(-z)] = q(1-q)$, which is concave on $[0,1]$. As a consequence, the inequality
    \[
        (q-q')\leq c_\psi^2\left(2V_{\psi}(\frac{q + q'}{2})-V_{\psi}(q)-V_{\psi}(q')\right)
    \]    
    holds for all $q$, $q'\in \mathbb{R}$ with $c_\psi = \sqrt{2}$.
    
    According to Theorem~\ref{lemma-appen:ISC-OVR-square_hinge}, the excess risk w.r.t. 0-1 loss function over $P_{te}$ is bounded by that of the OVR method,
    \[
        R(T(\bm{f}(\cdot))) -  R^*\leq \sqrt{2\Big(\mathbb{E}_{(\x,y)\sim P_{te}}[\Psi_y(\bm{f}(\x))]-\inf_{\bm{f}'} \mathbb{E}_{(\x,y)\sim P_{te}}[\Psi_y(\bm{f}'(\x))]\Big)} = \sqrt{2\Big(R_\psi(\bm{f}) - R_\psi^*\Big)},
    \]
         where $R_\psi^* = \inf_{\bm{f}'}R_{\psi}(\bm{f}')$. Then by applying the equality of the risk of OVR strategy $R_\psi$ and that of our approach $R_{LAC}$ from Proposition~\ref{prop:LACU-equal-nonconvex}, we complete the proof.
\end{proof}

\subsection{Proofs of Theorem~\ref{thm:generalization-bound} and Theorem~\ref{thm:estimation-error}}
We first analyze the generalization error of the \textsc{Eulac} approach with kernel-based hypothesis space (Theorem~\ref{thm:generalization-bound}), based on which we provide the proof of the estimation error (Theorem~\ref{thm:estimation-error}) is presented.
\label{sec-appendix:proof-FSC}
\begin{proof}[Proof of Theorem~\ref{thm:generalization-bound}]
Recall that
\begin{equation}
R_{LAC}= \theta\underbrace{\mathbb{E}_{(\mathbf{x},y)\sim P_\textsf{kc}}\left[f_{\textsf{ac}}(\mathbf{x})-f_y(\mathbf{x})\right]}_{\coloneqq R_{LAC}^A} + \underbrace{\mathbb{E}_{\mathbf{x}\sim p_{X}^{te}(\x)}\left[\psi(f_{\textsf{ac}}(\mathbf{x}))+\sum_{k=1}^K\psi(-f_{k}(\mathbf{x}))\right]}_{\coloneqq R_{LAC}^B}\\,
\end{equation}
and 
\begin{equation}
\hat{R}_{LAC} = \theta\underbrace{\frac{1}{n_l}\sum_{i=1}^{n_l}\left( f_{\textsf{ac}}(\mathbf{x}_i)-f_{y_i}(\mathbf{x}_i)\right)}_{\coloneqq \hat{R}_{LAC}^A}+\underbrace{\frac{1}{n_u}\sum_{i=1}^{n_u} \left(\psi(f_{\textsf{ac}}(\mathbf{x}_i))+\sum_{k=1}^K\psi(-f_{k}(\mathbf{x}_i))\right)}_{\coloneqq \hat{R}_{LAC}^B}.
\end{equation}
To obtain the generalization bound of $\hat{R}_{LAC}$, it suffices to upper bound $\hat{R}_{LAC}^A$ and $\hat{R}_{LAC}^B$. 

Firstly, we study the generalization bound of $\hat{R}_{LAC}^A$. With the kernel-based hypothesis set $\mathcal{F} = \{\mathbf{x}\mapsto\langle\mathbf{w},\Phi(\mathbf{x})\rangle\ \vert\ \Vert\mathbf{w}\Vert_{\mathbb{F}}\leq\Lambda\}$ and $\kappa(\x,\x)\leq r^2$, according to McDiarmid's inequality and the standard analysis for generalization bound based on Rademacher complexity~\citep[Theorem 3.1]{book'12:foundation}, we have,
\begin{equation}
\label{eqproof:thm3-1}
R_{LAC}^A(f_1,\dots,f_K,f_{\textsf{ac}})\leq\hat{R}_{LAC}^A(f_1,\dots,f_K,f_{\textsf{ac}})+2\hat{\mathfrak{R}}_{D_L}(\widetilde{\mathcal{F}})+6\Lambda r\sqrt{\frac{2\log(2/\delta')}{n_l}},
\end{equation}
holds with probability at least $1-\delta'$, where $$\widetilde{\mathcal{F}}=\{(\mathbf{x},y)\mapsto \langle\mathbf{w}^{\textsf{ac}},\Phi(\mathbf{x})\rangle-\langle\mathbf{w}^{y},\Phi(\mathbf{x})\rangle \ \vert\ \Vert\mathbf{w}^1\Vert_{\mathbb{F}},\dots,\Vert\mathbf{w}^K\Vert_{\mathbb{F}},\Vert\mathbf{w}^{\textsf{ac}}\Vert_{\mathbb{F}}\leq \Lambda\}.$$ 

The Rademacher complexity of hypothesis set $\hat{\mathfrak{R}}_{D_L}(\widetilde{\mathcal{F}})$ can be further bounded by,

\begin{align*}
  \hat{\mathfrak{R}}_{D_L}(\widetilde{\mathcal{F}}) &= \frac{1}{n_l}\mathbb{E}_{\bm{\sigma}}\left[\sup_{f_1,\dots,f_K,f_{\textsf{ac}}\in\mathcal{F}}\sum_{i=1}^{n_l}\sigma_i(f_{\textsf{ac}}(\mathbf{x}_i)-f_{y_i}(\mathbf{x}_i))\right]\\
  & \leq \frac{1}{n_l}\mathbb{E}_{\sigma}\left[\sup_{f_{\textsf{ac}}\in\mathcal{F}}\sum_{i=1}^{n_l}\sigma_i f_{\textsf{ac}}(\mathbf{x}_i)\right]+\frac{1}{n_l}\mathbb{E}_{\sigma}\left[\sup_{f_1,\dots,f_K\in\mathcal{F}}\sum_{i=1}^{n_l}\sigma_i f_{y_i}(\mathbf{x}_i)\right]\\
  & = \hat{\mathfrak{R}}_{D_L}(\mathcal{F}) + \frac{1}{n_l}\mathbb{E}_{\sigma}\left[\sup_{f_1,\dots,f_K\in\mathcal{F}}\sum_{i=1}^{n_l}\sum_{j\in[K]}\sigma_i f_j(\mathbf{x}_i)\cdot\indicator(y_i=j)\right]\\
  &\leq \hat{\mathfrak{R}}_{D_L}(\mathcal{F}) + \frac{1}{n_l}\sum_{j\in[K]}\mathbb{E}_{\sigma}\left[\sup_{f_j\in\mathcal{F}}\sum_{i=1}^{n_l}\sigma_i f_j(\mathbf{x}_i)\cdot\indicator(y_i=j)\right]\\
  & = \hat{\mathfrak{R}}_{D_L}(\mathcal{F}) + \frac{1}{n_l}\sum_{j\in[K]}\mathbb{E}_{\sigma}\left[\sup_{f_j\in\mathcal{F}}\sum_{i=1}^{n_l}\sigma_i f_j(\mathbf{x}_i)\cdot\left(\frac{2\indicator(y_i=j)-1}{2}+\frac{1}{2}\right)\right]\\
  & \leq \hat{\mathfrak{R}}_{D_L}(\mathcal{F}) + \frac{1}{n_l}\sum_{j\in[K]}\mathbb{E}_{\sigma}\left[\sup_{f_j\in\mathcal{F}}\sum_{i=1}^{n_l}\sigma_i f_j(\mathbf{x}_i)\cdot\frac{2\indicator(y_i=j)-1}{2}\right] \\
  & \quad+ \frac{1}{n_l}\sum_{j\in[K]}\mathbb{E}_{\sigma}\left[\sup_{f_j\in\mathcal{F}}\sum_{i=1}^{n_l}\frac{1}{2}\sigma_i f_j(\mathbf{x}_i)\right]\\
  & = \hat{\mathfrak{R}}_{D_L}(\mathcal{F}) + \frac{K}{2}\hat{\mathfrak{R}}_{D_L}(\mathcal{F}) + \frac{K}{2}\hat{\mathfrak{R}}_{D_L}(\mathcal{F})\\
  & = (K+1)\hat{\mathfrak{R}}_{D_L}(\mathcal{F})
\end{align*}
By Theorem 5.5 of~\citet{book'12:foundation}, the Rademacher complexity of the kernel-based hypothesis set is bounded by $\frac{\Lambda r}{\sqrt{n_l}}$, i.e.,  $\hat{\mathfrak{R}}_{D_L}(\mathcal{F})\leq\frac{\Lambda r}{\sqrt{n_l}}$. As a consequence, we can get the generalization bound
\begin{equation}
\label{eq:generalization-LACU-A}
R_{LAC}^A(f_1,\dots,f_K,f_{\textsf{ac}})\leq\hat{R}_{LAC}^A(f_1,\dots,f_K,f_{\textsf{ac}})+\frac{2(K+1)\Lambda r}{\sqrt{n_l}}+6\Lambda r\sqrt{\frac{2\log(2/\delta')}{n_l}},
\end{equation}
which holds with probability at least $1-\delta'$ for all $f_1,\dots,f_K,f_{\textsf{ac}}\in\mathcal{F}$. 

Next, we turn to bound the term $R_{LAC}^B$. An argument similar to the one used to obtain~\eqref{eqproof:thm3-1} shows that
\begin{equation}
\label{eq:generalization-LACU-B-raw}
R_{LAC}^B(f_1,\dots,f_K,f_{\textsf{ac}})\leq \hat{R}_{LAC}^B(f_1,\dots,f_K,f_{\textsf{ac}}) + 2\hat{\mathfrak{R}}_{D_L}(\widetilde{\mathcal{F}}_{\Psi})+3(K+1)B_\psi\sqrt{\frac{\log(2/\delta')}{n_u}}
\end{equation}
holds with probability at least $1-\delta'$ for all $f_1,\dots,f_K,f_{\textsf{ac}}\in\mathcal{F}$, where $B_\psi = \sup_{a\in[-\Lambda r,\Lambda r]}\psi(a)$ and $\widetilde{\mathcal{F}}_{\Psi}=\{\mathbf{x}\mapsto\psi(f_{\textsf{ac}}(\mathbf{x}))-\sum_{k=1}^K \psi(-f_{k}(\mathbf{x}))\ \vert\ f_1,\dots,f_K,f_{\textsf{ac}}\in \mathcal{F}\}$. According to the Talagrand's comparison inequality~\citep{book/springer/oracle2011} and the fact $\hat{\mathfrak{R}}_{D_L}(\mathcal{F}_1+\mathcal{F}_2)\leq\hat{\mathfrak{R}}_{D_L}(\mathcal{F}_1)+\hat{\mathfrak{R}}_{D_L}(\mathcal{F}_2)$, we have,
\begin{equation}
\label{eq:R-LACU-B}
\hat{\mathfrak{R}}_{D_L}(\widetilde{\mathcal{F}}_{\Psi})\leq (K+1)L\hat{\mathfrak{R}}_{D_L}(\mathcal{F})\leq(K+1)L\frac{\Lambda r}{\sqrt{n_u}},
\end{equation}
where $L$ is the Lipschitz constant of surrogate loss function $\psi$. Notice that some surrogate loss functions, whose first order derivative is unbounded (like square loss), are not Lipschitz continuous on $\mathbb{R}$. However, since $f_1,\dots,f_K,f_{\textsf{ac}}$ are bounded in $[-\Lambda r,\Lambda r]$, the Talagrand's Lemma is still applicable, where the term $R_{LAC}^B(f_1,\dots,f_K,f_{\textsf{ac}})$ can also be bounded following the same argument.

Plugging ~\eqref{eq:R-LACU-B} into~\eqref{eq:generalization-LACU-B-raw}, we can get the generalization bound of $R_{LAC}^B$ that
\begin{equation}
\label{eq:generalization-LACU-B}
R_{LAC}^B(f_1,\dots,f_K,f_{\textsf{ac}})\leq \hat{R}_{LAC}^B(f_1,\dots,f_K,f_{\textsf{ac}}) + \frac{2(K+1)L\Lambda r}{\sqrt{n_u}}+3(K+1)B_\psi\sqrt{\frac{\log(2/\delta')}{n_u}},
\end{equation}
which holds with probability at least $1-\delta'$. 
Let $\delta' = \frac{\delta}{2}$ and sum~\eqref{eq:generalization-LACU-A} and~\eqref{eq:generalization-LACU-B} up, we can get that
\begin{equation}
\begin{aligned}
&R_{LAC}(f_1,\dots,f_{K},f_nc)\\
=\ &\theta\cdot R_{LAC}^A + R_{LAC}^B\\
\leq {}&\hat{R}_{LAC}(f_1,\dots,f_K,f_{\textsf{ac}})+\theta\cdot\Bigg(\frac{2(K+1)\Lambda r}{\sqrt{n_l}}+6\Lambda r\sqrt{\frac{2\log(4/\delta)}{n_l}}\Bigg) + \frac{2(K+1)L\Lambda r}{\sqrt{n_u}}\\
&\quad+3(K+1)B_\psi\sqrt{\frac{\log(4/\delta)}{n_u}}\\
\leq {}&\hat{R}_{LAC}(f_1,\dots,f_K,f_{\textsf{ac}})+\frac{2(K+1)\Lambda r}{\sqrt{n_l}}+6\Lambda r\sqrt{\frac{2\log(4/\delta)}{n_l}} + \frac{2(K+1)L\Lambda r}{\sqrt{n_u}}\\
&\quad+3(K+1)B_\psi\sqrt{\frac{\log(4/\delta)}{n_u}}\\
\end{aligned}
\end{equation}
holds with probability at least $1-\delta$, which finishes the proof.
\end{proof}

\begin{proof}[Proof of Theorem~\ref{thm:estimation-error}]
Recall that optimization problem~\eqref{eq:LACU-ERM} is formulated as,
\begin{equation*}
\min_{f_1,\dots,f_k,f_{\textsf{ac}}\in\mathbb{F}} \hat{R}_{LAC}(f_1,\dots,f_K,f_{\textsf{ac}})+\lambda\left(\sum_{k=1}^K \Vert f_k\Vert^2_{\mathbb{F}}+\Vert f_{\textsf{ac}}\Vert^2_{\mathbb{F}}\right),
\end{equation*}
whose regularization path (the set of solutions to these problems with varying regularization parameter $\lambda$) is identical to the corresponding optimization problem parameterized in terms of the constraint on RKHS norm. Therefore, the optimal solution $(\hat{f}_1,\dots,\hat{f}_K,\hat{f}_{\textsf{ac}})$ of~\eqref{eq:LACU-ERM} with certain $\lambda>0$ is also the solution for
\begin{equation}
\label{eq-appen:optimization_constrain}
    \begin{split}
    &\min_{f_1,\dots,f_K,f_{\textsf{ac}}\in\mathbb{F}} \hat{R}_{LAC}(f_1,\dots,f_K,f_{\textsf{ac}})\\
    &\qquad \ \mbox{s.t.} \quad\sum_{k=1}^K\Vert f_k\Vert^2_{\mathbb{F}}+\Vert f_{\textsf{ac}}\Vert_{\mathbb{F}}^2\leq c_\lambda^2, 
    \end{split}
\end{equation}
with certain $c_\lambda>0$. 

Thus, to obtain the estimation error bound of~\eqref{eq:LACU-ERM}, it is equal to consider the optimization problem~\eqref{eq-appen:optimization_constrain}, where $\hat{R}_{LAC}$ is minimized on the hypothesis set $\mathscr{F}=\{(f_1,\dots,f_K,f_{\textsf{ac}})\ \vert\ f_1,\dots,f_K,f_{\textsf{ac}}\in\mathbb{F},\sum_{k=1}^K\Vert f_k\Vert^2_\mathbb{F}+\Vert f_{\textsf{ac}}\Vert_\mathbb{F}^2\leq c_\lambda^2\}$. Since $(f_1^*,\dots,f_K^*,f_{\textsf{ac}}^*)=\argmin_{\bm{f}\in\mathscr{F}}R_{LAC}(f_1,\dots,f_K,f_{\textsf{ac}})$, where $\bm{f} = (f_1,\dots,f_K,f_{\textsf{ac}})$, we have 
\begin{equation*}
    \begin{split}
          {} & R_{LAC}(\hat{f}_1,\dots,\hat{f}_{K},\hat{f}_{\textsf{ac}}) -  R_{LAC}(f_1^*,\dots,f_{K}^*,f_{\textsf{ac}}^*)\\
        = {} & R_{LAC}(\hat{f}_1,\dots,\hat{f}_{K},\hat{f}_{\textsf{ac}}) - \hat{R}_{LAC}(\hat{f}_1,\dots,\hat{f}_{K},\hat{f}_{\textsf{ac}}) + \hat{R}_{LAC}(\hat{f}_1,\dots,\hat{f}_{K},\hat{f}_{\textsf{ac}}) - R_{LAC}(f_1^*,\dots,f_{K}^*,f_{\textsf{ac}}^*)\\
        \leq {} & R_{LAC}(\hat{f}_1,\dots,\hat{f}_{K},\hat{f}_{\textsf{ac}}) - \hat{R}_{LAC}(\hat{f}_1,\dots,\hat{f}_{K},\hat{f}_{\textsf{ac}}) + \hat{R}_{LAC}(f_1^*\dots,f_{K}^*,f_{\textsf{ac}}^*) - R_{LAC}(f_1^*,\dots,f_{K}^*,f_{\textsf{ac}}^*)\\
        \leq {} & 2\sup_{\bm{f}\in\mathscr{F}} \vert R_{LAC}(f_1,\dots,f_K,f_{\textsf{ac}})-\hat{R}_{LAC}(f_1,\dots,f_K,f_{\textsf{ac}})\vert
    \end{split}
\end{equation*}

Let $\mathcal{F}_\lambda = \{\mathbf{x}\mapsto\langle\mathbf{w},\Phi(\mathbf{x})\rangle\ \vert\ \Vert\mathbf{w}\Vert_{\mathbb{F}}\leq c_\lambda\}$. Since $\mathscr{F}\subset\mathcal{F}^{K+1}_\lambda$, we have
\begin{equation*}
    \begin{split}
        	 {} & R_{LAC}(\hat{f}_1,\dots,\hat{f}_{K},\hat{f}_{\textsf{ac}}) -  R_{LAC}(f_1^*,\dots,f_{K}^*,f_{\textsf{ac}}^*)\\
        \leq {} & 2\sup_{\bm{f}\in\mathcal{F}^{K+1}_\lambda} \vert R_{LAC}(f_1,\dots,f_K,f_{\textsf{ac}})-\hat{R}_{LAC\textbf{}}(f_1,\dots,f_K,f_{\textsf{ac}})\vert,
    \end{split}
\end{equation*}
where $\mathcal{F}_{\lambda}^{K+1}$ refers to the Cartesian product over $\mathcal{F}_{\lambda}$ for $K$ times.

According to Theorem~\ref{thm:generalization-bound} and its counterpart which upper bounds $\hat{R}_{LAC}(f_1,\dots,f_K,f_{\textsf{ac}})-R_{LAC}(f_1,\dots,f_K,f_{\textsf{ac}})$, we have
\begin{align*}
       	 {} & 2\sup_{\bm{f}\in\mathcal{F}^{K+1}_\lambda} \vert R_{LAC}(f_1,\dots,f_K,f_{\textsf{ac}})-\hat{R}_{LAC}(f_1,\dots,f_K,f_{\textsf{ac}})\vert\\
    \leq {} & \frac{4(K+1)c_\lambda r}{\sqrt{n_l}}+6c_\lambda
    r\sqrt{\frac{2\log(8/\delta)}{n_l}}+ \frac{4(K+1)Lc_\lambda r}{\sqrt{n_u}}+12(K+1)B_\psi\sqrt{\frac{\log(8/\delta)}{n_u}},
\end{align*}
    which holds with probability at least $1-\delta$. The above argument completes the proof.
\end{proof}

\section{Details of Experiments}
\label{sec-appendix:experiments}
\subsection{Comparison on RKHS-based \textsc{Eulac}}
\label{sec-appendix:experiments-1}
\textbf{Datasets.} We first provide statistics of datasets used in the experiments in Table~\ref{table:dataset}. For all datasets, we normalize features of the raw data in the range of $[0,1]$. Datasets usps, segment, satimage, pendigits, SensIT Vehicle (SenseVeh), mnist and shuttle can be download from~\url{https://www.csie.ntu.edu.tw/~cjlin/libsvmtools/datasets/}. Optdigits and landsets are from~\url{https://archive.ics.uci.edu/ml/index.php}.

For each dataset, the labeled training data, unlabeled training data and testing data contain 500, 1000 and 1000 instances respectively. The instance sampling procedure also repeats 10 times. To simulate the augmentation of classes, we randomly select half of the total classes as augmented for 10 times. Thus, there are in total 100 configurations for each datasets. We remark that the testing data are never used in the training stage.

\begin{table}[!h]
\centering
\caption{Statistics of datasets used in the experiments.}
\vspace{-1mm}
\label{table:dataset}
\resizebox{0.6\textwidth}{!}{
\begin{tabular}{clrrrrr}
\toprule
\multirow{2}{*}{Index} & \multirow{2}{*}{Datasets} & \multirow{2}{*}{\# class} & \multirow{2}{*}{\# dim} & \multicolumn{3}{c}{$\vert C \vert$} \\ \cline{5-7} 
           &               &                    &                            & min        & max        & avg       \\ \midrule
1 & usps                      & 10                 & 256                        & 708        & 1553       & 929       \\
2 & segment                   & 7                  & 19                         & 330        & 330        & 330       \\
3 & satimage                  & 6                  & 36                         & 626        & 1533       & 1073      \\
4 & optdigits                 & 10                 & 64                         & 554        & 572        & 562       \\
5 & pendigits                 & 10                 & 16                         & 1055       & 1144       & 1099      \\
6 & SenseVeh                  & 3                  & 100                        & 12316      & 26423      & 20527     \\
7 & landset                   & 6                  & 73                         & 626        & 1533       & 1073      \\
8 & mnist                     & 10                 & 780                        & 6313       & 7877       & 7000      \\
9 & shuttle                   & 7                  & 9                          & 10         & 45586      & 8286      \\ \bottomrule
\end{tabular}}
\end{table}

\textbf{Contenders and parameters setting.} We compare the \textsc{Eulac} with six methods, including four without exploiting unlabeled data and two utilizing them.
\begin{itemize}
      \item \textbf{\textsc{Eulac}} is the basic version of our proposal, which minimizes LAC risk $R_{LAC}$~\eqref{eq:expect-PU} on the kernel-based hypothesis set with the squared loss $\psi(z) = (1-z)^2/4$. We adapt the Gaussian kernel $\kappa(\mathbf{x},\mathbf{y}) = \exp(\Vert\mathbf{x}-\mathbf{y}\Vert^2/2\sigma^2)$, where the bandwidth $\sigma$ is selected from $\{10^{-2},\dots,10\}\times\mbox{median}_{i,j\in[n_l+n_u]}(\Vert\mathbf{x}_i-\mathbf{x}_j\Vert)$ and the regularization parameters $\{10^{-3},\dots,10^1\}$ by cross validation. Meanwhile, we use KME-base to estimate the mixture proportion, where the threshold $\nu$ for the KME-based MPE estimator (Algorithm~\ref{alg:KME}) is set to 0.25.
\end{itemize}

The four methods without exploiting unlabeled data are,
\begin{itemize}
    \item \textbf{OVR-SVM} is a powerful strategy for multi-class classification~\citep{journals/jmlr/RifkinK03}. In order to adapt OVR-SVM to LAC problem, the method predicts an instance belonging to the augmented class when $\max_{k\in[K]} f_k<0$, otherwise it predicts as the classical OVR-SVM. The parameters setting is the same as that in \textsc{Eulac}, except that the bandwidth is selected from $\{10^{-2},\dots,10\}\times\mbox{median}_{i,j\in[n_l]}(\Vert\mathbf{x}_i-\mathbf{x}_j\Vert)$ as there are no unlabeled data.

    \item \textbf{W-SVM}~\citep{journals/pami/ScheirerJB14} is an SVM-based method, where both one-class SVM and binary SVM incorporating with extreme value theory are used to predict for the augmented class. The parameters setting is the same as that in \textsc{Eulac}, except that the bandwidth is selected from $\{10^{-2},\dots,10\}\times\mbox{median}_{i,j\in[n_l]}(\Vert\mathbf{x}_i-\mathbf{x}_j\Vert)$ as there are no unlabeled data.

    \item \textbf{OSNN}~\citep{journals/ml/Mendes-JuniorSW17} is a nearest neighbor-based method, which predicts an instance belonging to augmented class if it shares similar distances with two nearest neighbors from different classes. We set parameters according to the corresponding paper.
    \item \textbf{EVM}~\citep{journals/pami/RuddJSB18} is also based on the extreme value theory, and it uses non-linear radial basis functions. We set parameters according to the corresponding paper.
\end{itemize}

We list other two methods exploiting unlabeled data in the following.
\begin{itemize}
    \item \textbf{LACU-SVM}~\citep{conf/aaai/DaYZ14} is an SVM based method that utilizes the geometry property of unlabeled data to tune the decision boundaries of classifiers. We set parameters according to the corresponding paper.
    \item \textbf{PAC-iForest}~\citep{conf/icml/LiuGDFH18} is an iForest~\citep{conf/icdm/LiuTZ08} based method, which selects the rejection threshold by using unlabeled data to ensure desired novelty detection ratio. We use PAC-iForest to detect augmented classes and SVM to classify known classes. The novel detection ratio is set to 0.7 for all datasets.
\end{itemize}

\textbf{Performance Measure.} We use three measures to evaluate the performance of contenders, including accuracy, Macro-F1 and AUC of the augmented classes score. Denoting by $D_{te} = \{(\x_i,y_i)\}_{i=1}^{n_{te}}$ the testing dataset including augmented classes (\textsf{ac}), and $\hat{y}_i = f(\x_i)$ the prediction for the $i$-th instance, the first two measures focus on the \emph{overall performance} as,
\begin{itemize}
    \item \textbf{Accuracy}: the averaged predictive accuracy on testing data, $$\mathrm{Acc} = \frac{1}{n_{te}}\sum_{i=1}^{n_{te}}\indicator({y_i=\hat{y}_i});$$
    \item \textbf{Macro-F1}: the averaged F1 score including that of the augmented class, $$\text{Macro-F1} = \frac{1}{K+1}\bigg(F_{\textsf{ac}}+\sum_{k=1}^K F_k\bigg),$$ where $F_k = (2\times P_k\times R_k)/(P_k+R_k)$ with $P_k$ ($R_k$) the precise (recall) of the $k$-th class. Similar notations ($F_{\textsf{ac}}$, $P_{\textsf{ac}}, R_{\textsf{ac}}$) are defined for the augmented class.
\end{itemize}

\textbf{AUC of identifying the augmented class.} We also report the AUC of contenders' augmented class score for comparing their ability for handling the augmented class. The reason for this measure is that most compared methods handling the LAC problem following the three steps: (1) assigning a score for each instance to evaluate its probability for belonging to the augmented classes; (2) specifying a threshold and identifying the augmented class; (3) classifying the rest instances to known classes. In step 2, the compared methods always require to set the threshold empirically, where a misspecified one would lead to performance degeneration. In contrast to the contenders, our approach can perform an unbiased cross validation to select the parameters. Thus, to make further comparison and ablate the influence on misspecified threshold, we employ AUC to measure the quality of the augmented class score returned by the compared methods. 

For different compared methods, we list the way of computing the augmented class score to measure how possible a sample could be the augmented class sample.
\begin{itemize}
  \item \textbf{OVR-SVM} predicts augmented class when $\max_{k\in[K]}f_k(\x)\leq 0$. So, it is natural to use $\max_{k\in[K]}f_k(\x)$ as a score to measure how possible a sample coming from the augmented class.
  \item \textbf{OSNN} is a nearest neighbor based method, denoting by $\x_i$ the testing sample, its augmented class score is calculate by
  $$
    R(\x_i) = \mathrm{dist}(\x_i,\x_t) / \mathrm{dist}(\x_i,\x_u),
  $$
  where $\x_t$ is the nearest sample of $\x_i$ in the training set under the distance $\mathrm{dist}:\X\times\X\mapsto\mathbb{R}$ and $\x_u$ is the rest nearest sample whose label is different from $\x_t$.

  \item \textbf{EVM} classifies a sample $\x_i$ as the augmented class when its empirically maximum conditional probability $\hat{P}(k|\x_i)=\argmax_{\{i:y_i=k\}}\Psi(\x_i)$ is less than a given threshold $\delta>0$, where $\Psi:\X\mapsto\mathbb{R}$ is the density function learned in training phase. Therefore, the augmented class score is given by $\hat{P}(k|\x_i)$.
  \item \textbf{LACU-SVM} is based on OVR-SVM, it adjusts the margin of OVR-SVM with unlabeled data. Similar to OVR-SVM, the augmented class score of LACU-SVM is also the maximum output of probabilities for each class, which is given by $\max_{k\in[K]} f_k(\x)$.
  \item \textbf{PAC-iForest} outputs the path length of each sample in iForest and thus the augmented class score of each sample is the path length in iForest.
  \item \textbf{\textsc{Eulac}:} Although \textsc{Eulac} does not identify the augmented class following the aforementioned three steps, to make comparisons, we take the output of the classifier $f_{\textsf{ac}}(\x)$ as the augmented class score, as $f_{\textsf{ac}}$ is the binary classifier for the augmented class.

  \item \textbf{W-SVM:} We do not consider the AUC for W-SVM. The reason is that W-SVM takes a two-step mechanism to filter the augmented classes and thus there are two `scores' that decide the labeling. As a result, varying either of the two `scores' could change recalls from zero to complete, and the corresponding ROC curve will no longer be monotonic, which means it is improper. In this case, AUC has lost its significance as the calibration-free measure of the performance of W-SVM.
\end{itemize}

\textbf{More empirical results.} We provide more performance comparisons of all methods on benchmark datasets. Table~\ref{table:experiment1-Accuracy} reports the accuracy of all contenders and Table~\ref{table:experiment1-AUC} presents the AUC of all contenders.
\begin{table}[!t]
\centering
\caption{Accuracy on benchmark datasets. The best method is emphasized in bold. Besides, $\bullet$ indicates that \textsc{Eulac} is significantly better than others (paired $t$-tests at 5\% significance level).}
\vspace{-1mm}
\label{table:experiment1-Accuracy}
\resizebox{1\textwidth}{!}{
\begin{tabular}{lllllllll}
\toprule
\multirow{1}{*}{Dataset}                     & \multicolumn{1}{c}{\multirow{1}{*}{OVR-SVM}}  & \multicolumn{1}{c}{\multirow{1}{*}{W-SVM}} & \multicolumn{1}{c}{\multirow{1}{*}{OSNN}}  & \multicolumn{1}{c}{\multirow{1}{*}{EVM}}   & \multicolumn{1}{c}{\multirow{1}{*}{LACU-SVM}} & \multicolumn{1}{c}{PAC-iForest}      & \multicolumn{1}{c}{\multirow{1}{*}{\textsc{Eulac}}} \\ \midrule
usps                                         & 72.94  $\pm$ 3.91 $\bullet$                   & 86.00 $\pm$ 4.32 $\bullet$                & 83.26 $\pm$ 2.15  $\bullet$                 & 54.30 $\pm$ 3.27  $\bullet$                & 75.44 $\pm$ 6.32 $\bullet$                & 53.43 $\pm$ 3.08   $\bullet$             & \textbf{90.00 $\pm$ 1.01}        \\
segment                                      & 73.36  $\pm$ 7.56 $\bullet$                   & 87.05 $\pm$ 8.23 $\bullet$                & 84.51 $\pm$ 7.57  $\bullet$                 & 79.60 $\pm$ 4.81  $\bullet$                & 51.03 $\pm$ 10.3 $\bullet$                & 54.83 $\pm$ 4.08   $\bullet$             & \textbf{89.33 $\pm$ 2.21}       \\
satimage                                     & 52.52  $\pm$ 11.3 $\bullet$                   & 82.15 $\pm$ 11.2 $\bullet$                & 60.93 $\pm$ 9.95  $\bullet$                 & 61.50 $\pm$ 9.35  $\bullet$                & 61.87 $\pm$ 14.7 $\bullet$                & 56.71 $\pm$ 9.90   $\bullet$             & \textbf{87.44 $\pm$ 3.10}       \\
optdigits                                    & 74.16  $\pm$ 2.45 $\bullet$                   & 91.53 $\pm$ 3.74 $\bullet$                & 89.02 $\pm$ 2.58  $\bullet$                 & 71.40 $\pm$ 1.57  $\bullet$                & 80.59 $\pm$ 3.20 $\bullet$                & 57.23 $\pm$ 1.11   $\bullet$             & \textbf{93.73 $\pm$ 1.01}       \\
pendigits                                    & 77.09  $\pm$ 3.25 $\bullet$                   & 90.87 $\pm$ 4.14                      & 87.86 $\pm$ 2.68  $\bullet$                 & 76.50 $\pm$ 2.16  $\bullet$                & 73.18 $\pm$ 4.91 $\bullet$                & 59.89 $\pm$ 1.58   $\bullet$             & 90.82 $\pm$ 1.33        \\
SenseVeh                                     & 59.80  $\pm$ 5.88 $\bullet$                   & 57.00 $\pm$ 6.12 $\bullet$                & 51.30 $\pm$ 5.69  $\bullet$                 & 32.00 $\pm$ 0.40  $\bullet$                & 54.16 $\pm$ 1.57 $\bullet$                & 64.10 $\pm$ 3.62   $\bullet$             & \textbf{67.78 $\pm$ 15.6}       \\
landset                                      & 51.90  $\pm$ 7.49 $\bullet$                   & 82.89 $\pm$ 10.2 $\bullet$                & 75.50 $\pm$ 7.40  $\bullet$                 & 60.10 $\pm$ 6.20  $\bullet$                & 61.07 $\pm$ 8.87 $\bullet$                & 56.45 $\pm$ 7.27   $\bullet$             & \textbf{90.22 $\pm$ 1.62}       \\
mnist                                        & 65.30  $\pm$ 3.69 $\bullet$                   & 84.32 $\pm$ 3.21                          & 76.85 $\pm$ 3.50  $\bullet$                 & 52.80 $\pm$ 3.08  $\bullet$                & 71.55 $\pm$ 4.47 $\bullet$                & 52.71 $\pm$ 3.54   $\bullet$             & 84.04 $\pm$ 2.06       \\
shuttle                                      & 45.59  $\pm$ 22.6 $\bullet$                   & 66.05 $\pm$ 38.9 $\bullet$                & 63.25 $\pm$ 28.2  $\bullet$                 & \multicolumn{1}{c}{--}                & 65.98 $\pm$ 20.0 $\bullet$                & 56.47 $\pm$ 23.8   $\bullet$             & \textbf{96.52 $\pm$ 3.07}       \\ \midrule
\textsc{Eulac} w/ t/ l                                & \multicolumn{1}{c}{9/ 0/ 0}                   & \multicolumn{1}{c}{7/ 2/ 0}               & \multicolumn{1}{c}{9/ 0/ 0}                 & \multicolumn{1}{c}{9/ 0/ 0}                & \multicolumn{1}{c}{9/ 0/ 0}               & \multicolumn{1}{c}{9/ 0/ 0}              & rank first 7/ 9       \\ 
\bottomrule       
\end{tabular}
}
\end{table}

\begin{table}[!t]
\centering
\caption{AUC on benchmark datasets. The best method is emphasized in bold. Besides, $\bullet$ indicates that \textsc{Eulac} is significantly better than others (paired $t$-tests at 5\% significance level).}
\vspace{-1mm}
\label{table:experiment1-AUC}
\resizebox{1\textwidth}{!}{
\begin{tabular}{llllllll}
\toprule
\multirow{1}{*}{Dataset} & \multicolumn{1}{c}{\multirow{1}{*}{OVR-SVM}}                &  \multicolumn{1}{c}{\multirow{1}{*}{OSNN}}      & \multicolumn{1}{c}{\multirow{1}{*}{EVM}}    & \multicolumn{1}{c}{\multirow{1}{*}{LACU-SVM}} & \multicolumn{1}{c}{PAC-iForest}      & \multicolumn{1}{c}{\multirow{1}{*}{\textsc{Eulac}}} \\ \midrule
usps                                         & 36.52 $\pm$  4.70  $\bullet$            & 90.73 $\pm$ 3.03  $\bullet$                     & 68.75 $\pm$ 6.20 $\bullet$                  & 14.05 $\pm$ 4.86  $\bullet$                    & 70.68 $\pm$ 12.7  $\bullet$         & \textbf{97.05 $\pm$ 1.08}       \\
segment                                      & 68.43 $\pm$  8.12  $\bullet$            & 92.77 $\pm$ 3.26  $\bullet$                     & 80.15 $\pm$ 9.45 $\bullet$                  & 44.15 $\pm$ 11.7  $\bullet$                    & 77.08 $\pm$ 12.7  $\bullet$       & \textbf{95.60 $\pm$ 4.01}       \\
satimage                                     & 43.22 $\pm$  10.7  $\bullet$            & 71.91 $\pm$ 13.2  $\bullet$                     & 78.77 $\pm$ 6.56 $\bullet$                  & 26.51 $\pm$ 12.2  $\bullet$                    & 86.88 $\pm$ 4.68  $\bullet$       & \textbf{94.92 $\pm$ 2.80}       \\
optdigits                                    & 25.95 $\pm$  2.40  $\bullet$            & 94.58 $\pm$ 2.11  $\bullet$                     & 82.01 $\pm$ 2.89 $\bullet$                  & 9.161 $\pm$ 2.21  $\bullet$                    & 82.82 $\pm$ 5.81  $\bullet$       & \textbf{99.08 $\pm$ 0.66}       \\
pendigits                                    & 34.66 $\pm$  3.39  $\bullet$            & 94.12 $\pm$ 2.68  $\bullet$                     & 89.91 $\pm$ 3.34 $\bullet$                  & 17.15 $\pm$ 4.15  $\bullet$                    & 87.93 $\pm$ 3.64  $\bullet$       & \textbf{97.27 $\pm$ 2.23}        \\
SenseVeh                                     & 42.14 $\pm$  9.38  $\bullet$            & 55.48 $\pm$ 8.65  $\bullet$                     & 65.79 $\pm$ 9.73 $\bullet$                  & 19.56 $\pm$ 9.74  $\bullet$                    & 71.02 $\pm$ 16.6  $\bullet$       & \textbf{88.61 $\pm$ 2.90}       \\
landset                                      & 44.32 $\pm$  11.1  $\bullet$            & 85.82 $\pm$ 8.48  $\bullet$                     & 80.62 $\pm$ 7.85 $\bullet$                  & 27.98 $\pm$ 16.8  $\bullet$                    & 89.76 $\pm$ 4.05  $\bullet$       & \textbf{97.32 $\pm$ 2.67}       \\
mnist                                        & 38.32 $\pm$  5.26  $\bullet$            & 85.53 $\pm$ 4.12  $\bullet$                     & 65.76 $\pm$ 5.98 $\bullet$                  & 17.32 $\pm$ 5.67  $\bullet$                    & 63.24 $\pm$ 9.27  $\bullet$       & \textbf{93.45 $\pm$ 2.99}       \\
shuttle                                      & 26.85 $\pm$  10.2  $\bullet$            & \textbf{97.52 $\pm$ 2.61}                       & \multicolumn{1}{c}{--}                  & 11.49 $\pm$ 13.5  $\bullet$                    & 93.20 $\pm$ 5.75                  & 79.77 $\pm$ 24.3       \\ \midrule
\textsc{Eulac} w/ t/ l                                & \multicolumn{1}{c}{9/ 0/ 0}              & \multicolumn{1}{c}{8/ 0/ 1}                    & \multicolumn{1}{c}{9/ 0/ 0}                 & \multicolumn{1}{c}{9/ 0/ 0}                    & \multicolumn{1}{c}{8/ 0/ 1}         & rank first 8/ 9       \\ 
\bottomrule
\end{tabular}
}
\end{table}

\subsection{Comparison on Deep models}
\label{sec-appendix:experiments-2}
Here we elaborate on how we implement DNN-based \textsc{Eulac}, including descriptions for datasets and contenders.

\textbf{Datasets and configurations.}
In the experiment, mnist, SVHN, Cifar-10 are used following the protocol in previous studies~\citep{conf/eccv/NealOFWL18}. All three datasets contain ten classes, among which we randomly choose six classes as known classes and four other classes as augmented classes. The data partition follows the standard split and meanwhile maintains the balance among different classes, which means that we subsample an equal number of instances for each class from the standard split. Detailed descriptions for datasets are as follows.
\begin{itemize}
\item \textbf{mnist} contains 28$\times$28 monochrome images of digital numbers. The training set of mnist experiment contains both instances of the labeled part and unlabeled part. The labeled part contains instances of the known six classes randomly picked in advance, with 6000 instances per class randomly subsampled from the standard split of the mnist dataset. In comparison, the unlabeled part contains unlabeled instances from 10 classes, with 1000 instances in each class. Thus the total number of instances in the training set is 46000. The testing set consists of ten classes with 100 instances in each class, and thus the total number of the testing set is 10000. The dataset can be downloaded from~\url{http://yann.lecun.com/exdb/mnist/}.
\item \textbf{SVHN} is short for Street View House Numbers, which includes 32$\times$32 colored image samples of numbers. Similar to previous experiments, the training set consists of the six known classes with 4500 subsampled instances in each class and ten unlabeled classes with 3000 subsampled instances in each class, and the total number of instances is 57000. The testing set contains instances of all ten classes with 100 instances in each class. SVHN dataset can be downloaded from~\url{http://ufldl.stanford.edu/housenumbers/}.
\item \textbf{Cifar-10} consists of ten classes of natural images of size 32$\times$32. The training set contains the six known classes with 5000 subsampled instances in each class and unlabeled instances from all ten classes with 900 subsampled instances in each class; thus, the total number of instances of the training set is 39000. The testing set contains all ten classes with 100 instances per class, as is described before. Cifar-10 dataset can be downloaded from~\url{https://www.cs.toronto.edu/~kriz/cifar.html}.
\end{itemize}
In all experiments, data used for evaluation are never used in the training stage (namely, the evaluation data are not contained in labeled and unlabeled training data).

\textbf{Contenders.}
Here we introduce the technical details of the compared approaches. 
\begin{itemize}
\item \textbf{Softmax Threshold} is the counterpart of OVR-SVM with deep model, where classifiers $f_1(\x),\dots,f_K(\x)$, the outputs of the neural network, are trained to approximate the posterior probabilities $p(y=k|\x)$, $k\in[K]$. An instance is predicted to the augmented class when the maximum outputs $\max_{k\in[K]} f_k(\x)$ is less than a given threshold.
\item \textbf{OpenMax}~\citep{conf/cvpr/BendaleB16} can be considered as a calibrated version of Softmax Threshold method. In this method, Weibull calibration is implemented as an augment to the SoftMax method while replacing the Softmax layer with a new OpenMax layer, which outputs calibrated posterior probabilities.
\item \textbf{Generative OpenMax}~\citep{conf/bmvc/GeDG17} can be simply concluded as an OpenMax using samples generated from GAN. Inspired by this method, other variants of Generative OpenMax include using an encoder-decoder network instead of GAN, as implemented in~\citet{conf/eccv/NealOFWL18}. 
\item \textbf{OSRCI}~\citep{conf/eccv/NealOFWL18} is the short for Open-set Recognition using Counterfactual Images (OSRCI). The method is similar to Generative OpenMax and is also an data-augmentation-based techniques. While unlike Generative OpenMax, this method uses a counterfactual image generation technique to train an additional $f_{K+1}$ classier for the augmented class.
\end{itemize}
\textbf{Parameters setting.} Since results for other compared methods are taken from~\citep{conf/eccv/NealOFWL18}, we only introduce our model details. For all datasets, we employ the sigmoid loss $\psi(z) = 1/(1+\exp(z))$ for training. The loss is optimized by Adam optimizer with learning rate $10^{-5}$ and weight decay $5 \times 10^{-3}$. Following~\citep{conf/nips/KiryoNPS17}, we use various network architectures for different datasets.

\begin{itemize}
  \item \textbf{mnist.} We use a convolutional fully-connected neural network whose architecture consists of two convolutional layers with 3$\times$3 kernels of 32 and 64 filters, followed by five hidden layers of 390 units each with dropout applied before every dense layer. The MLP is trained for 200 epochs.
  
  \item \textbf{SVHN.} We choose VGG-13 architecture with ReLU activation without dropout. Model parameters are trained for 400 epochs.
  
  \item \textbf{Cifar-10.} We choose VGG-16 architecture with batch normalization before every ReLU activation without dropout. Model parameters are trained for 400 epochs.
\end{itemize}

Unlike the RKHS-based \textsc{Eulac}, it is surprising that a relatively higher mixture proportion will help the neural networks converge better. We thus select the mixture proportion $\theta$ as the one achieving the minimal training loss from interval $[0.6,0.9]$ with the grid 0.05. We present the empirical observation on the relationship between the AUC and training loss as follows.

\paragraph{More empirical results.} Table~\ref{table:auc-loss} shows the training loss and AUC of networks trained with various mixture proportions. A large mixture proportion could help the networks converge better, leading to a higher AUC. Figure~\ref{fig:SVHN-loss} and Figure~\ref{fig:SVHN-auc} present the learning curve on SVHN as a more detailed illustration.

\begin{table}[!t]
\centering
\caption{The relationship between the training loss and AUC with various mixture proportions.}
\label{table:auc-loss}
\begin{tabular}{ccccccc}
\toprule
\multirow{2}{*}{$\theta$} & \multicolumn{2}{c}{mnist}       & \multicolumn{2}{c}{Cifar-10}    & \multicolumn{2}{c}{SVHN}        \\ \cmidrule{2-7} 
                                & training loss    & AUC       & training loss     & AUC       & training loss     & AUC      \\ \midrule
0.6                             & 82.9 $\pm$ 0.4 & 98.9 $\pm$ 0.4 & 95.7 $\pm$ 0.2 & 55.5 $\pm$ 1.6 & 78.4 $\pm$ 0.2 & 78.0 $\pm$ 0.3 \\
0.7                             & 79.9 $\pm$ 0.3 & 98.8 $\pm$ 0.5 & 91.8 $\pm$ 0.2 & 82.3 $\pm$ 1.4 & 73.0 $\pm$ 0.1 & 80.4 $\pm$ 3.4 \\
0.8                             & 76.4 $\pm$ 0.3 & 99.1 $\pm$ 0.4 & 81.2 $\pm$ 0.2 & 82.7 $\pm$ 1.0 & 69.2 $\pm$ 0.2 & 81.7 $\pm$ 3.1 \\
0.9                             & 55.2 $\pm$ 4.0 & 98.6 $\pm$ 0.6 & 66.2 $\pm$ 0.2 & 84.9 $\pm$ 2.0 & 45.4 $\pm$ 1.2 & 90.2 $\pm$ 2.8 \\ \bottomrule
\end{tabular}
\end{table}

\begin{figure}[!t]
\raggedright 
 \begin{minipage}[t]{0.49\textwidth}
    \centering
    \includegraphics[height=0.5\textwidth]{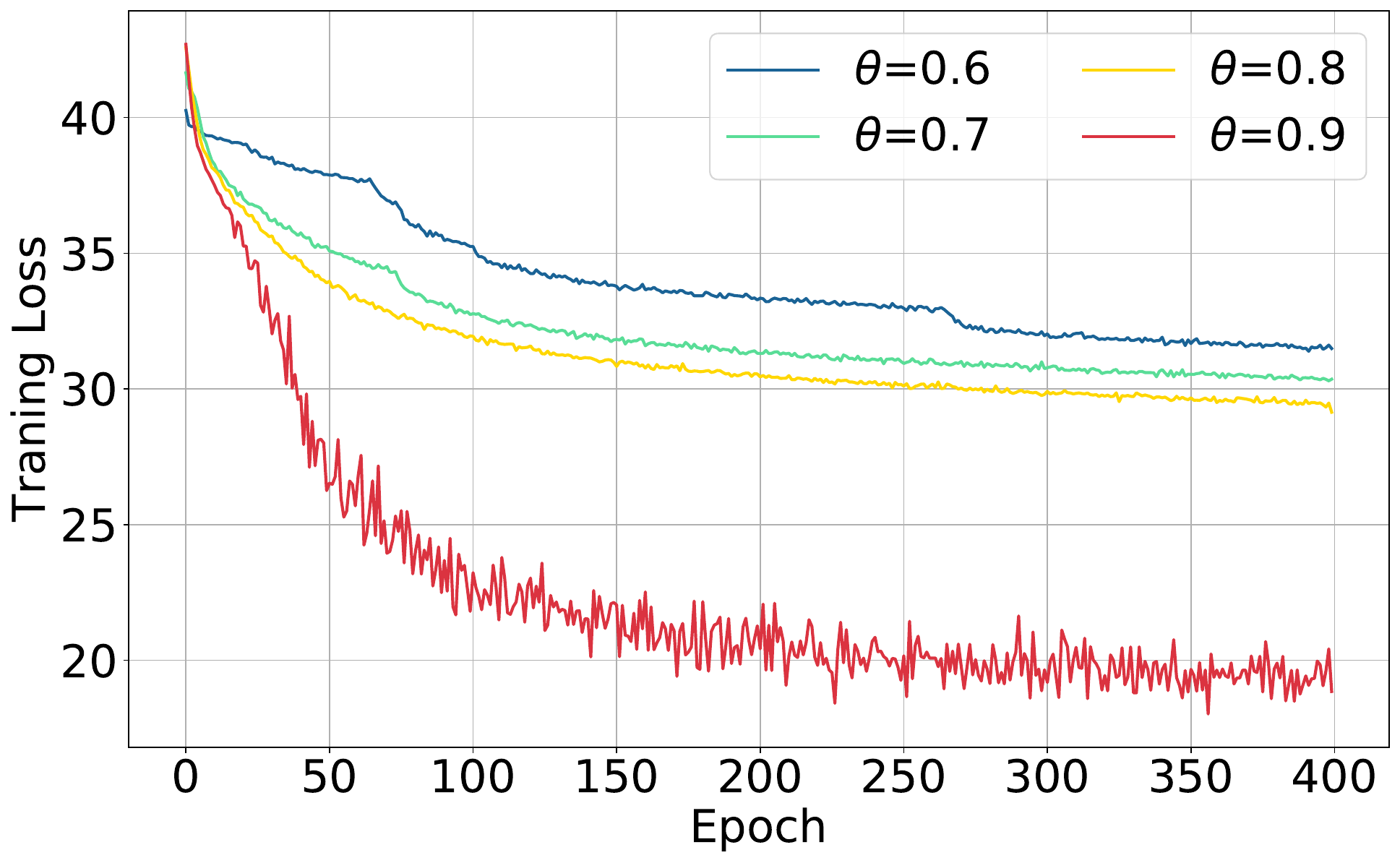}
    \setcaptionwidth{5.8 cm}
    \caption{Training loss on SVHN}
    \label{fig:SVHN-loss}
\end{minipage} 
\begin{minipage}[t]{0.49\textwidth}
    \centering
    \includegraphics[height=0.5\textwidth]{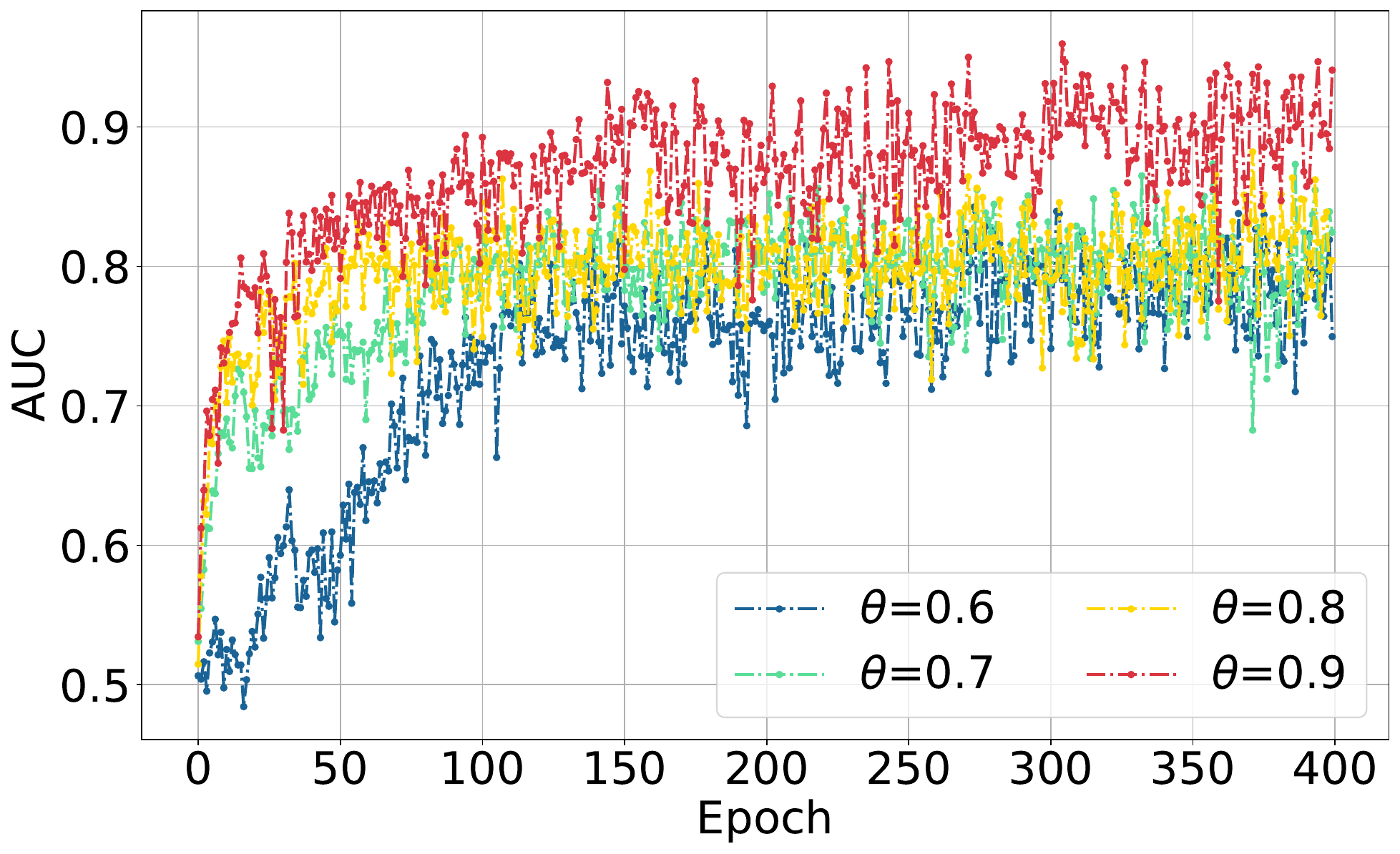}
    \setcaptionwidth{7 cm}
    \caption{AUC on SVHN}
    \label{fig:SVHN-auc}
\end{minipage}  
  \vspace{-3mm}
\end{figure}

\subsection{Comparisons in complex changing environments}
\label{sec-appendix:experiment-3}
In this part, we further describe contenders of experiments where the augmented classes appear with the distribution change on the prior of known classes. There are five compared methods, where two do not take the distribution change of priors into account (LACU-SVM, \textsc{Eulac}-base) and two only consider the distribution change for training classifiers of known classes (OVR-shift, \textsc{Eulac}-base++) and thus is biased. \textsc{Eulac}-shift is the unbiased estimator. For all contenders, we use KME-shift for estimating the mixture proportion $\theta$ or the class prior $\theta_{te}^k$, $k\in[K]$.

\textbf{Contenders.} The detailed descriptions for contenders are listed as follows. 
\begin{itemize}
  \item \textbf{LACU-SVM} is the original LACU-SVM method, which does not consider the prior change on known classes.
  \item \textbf{\textsc{Eulac}-base} is the basic version of our proposal, which minimizes the LAC risk $R_{LAC}$~\eqref{eq:expect-PU} on the kernel-based hypothesis set with the squared loss $\psi(z) = (1-z)^2/4$. This method does not take the prior change into account.
  \item \textbf{OVR-shift} is a variant of OVR-SVM for handling the distribution change of the prior on known classes, where instances from the known classes are reweighed to match their prior in the testing distribution. The classifiers are trained by minimizing the empirical version of
\begin{equation}
\label{eq:OVR-shift}
 \min_{f_1,\dots,f_{K}}\sum_{i=1}^K \theta_{te}^i\cdot \mathbb{E}_{\mathbf{x}\sim p_{X|Y}^\textsf{kc}(\x|i)}\Bigg[\psi(f_i(\mathbf{x}))+\sum_{k=1,k\neq i}^{K}\psi(-f_k(\mathbf{x}))\Bigg].
\end{equation}
The testing instance is predicted as augmented when $\max_{k\in[K]} f_k(\x)<0$. 
  \item \textbf{\textsc{Eulac}-base++} is a variant of \textsc{Eulac}-base, where OVR classifiers of the known classes $\{f_1,\dots,f_K\}$ are trained with reweighed instances for matching their prior in the testing distribution. The classifier of augmented classes ($f_{\textsf{ac}}$) is trained following the same scheme as \textsc{Eulac}-base.
  \item \textbf{\textsc{Eulac}-shift} is our proposal, which minimizes the LAC risk with the prior change $R_{LAC}^{shift}$ on the kernel-based hypothesis set. The method is unbiased over the testing distribution.
\end{itemize}
The parameters for all compared algorithms are the same as those in Appendix~\ref{sec-appendix:experiments-1}. We set the threshold $\nu=0.25$ for the KME-shift estimator.
\end{document}